\documentclass{article}
\usepackage{geometry}
\usepackage{graphicx} 
\usepackage{xcolor}
\title{Efficient and Provable Algorithms for \\ Covariate Shift}
\author{Deeksha Adil\\Institute for Theoretical Studies\\ETH Zürich \\ deeksha.adil@eth-its.ethz.ch\and  Jarosław Błasiok\\Department of Computer Science\\ ETH Zürich\\jaroslaw.blasiok@inf.ethz.ch
}

\usepackage{style}

\newcommand{\ptr}{p_{\mathrm{tr}}}
\newcommand{\pte}{p_{\mathrm{te}}}
\newcommand{\mtr}{\boldsymbol{\mathit{\mu}}_{\mathrm{tr}}}
\newcommand{\mte}{\boldsymbol{\mathit{\mu}}_{\mathrm{te}}}
\newcommand{\str}{\boldsymbol{\mathit{\Sigma}}_{\mathrm{tr}}}
\newcommand{\ste}{\boldsymbol{\mathit{\Sigma}}_{\mathrm{te}}}
\newcommand{\tte}{\boldsymbol{\mathit{\theta}}_{\mathrm{te}}}
\newcommand{\ttr}{\boldsymbol{\mathit{\theta}}_{\mathrm{tr}}}
\newcommand{\bb}{\boldsymbol{\mathit{b}}}
\newcommand{\ff}{\boldsymbol{\mathit{f}}}
\newcommand{\hh}{\boldsymbol{\mathit{h}}}
\newcommand{\ww}{\boldsymbol{\mathit{w}}}
\newcommand{\xx}{\boldsymbol{\mathit{x}}}
\newcommand{\yy}{\boldsymbol{\mathit{y}}}

\newcommand{\ttheta}{\boldsymbol{\mathit{\theta}}}

\renewcommand{\AA}{\boldsymbol{\mathit{A}}}
\newcommand{\II}{\boldsymbol{\mathit{I}}}

\begin{document}

\maketitle

\begin{abstract}
Covariate shift, a widely used assumption in tackling {\it distributional shift} (when training and test distributions differ), focuses on scenarios where the distribution of the labels conditioned on the feature vector is the same, but the distribution of features in the training and test data are different. Despite the significance and extensive work on covariate shift, theoretical guarantees for algorithms in this domain remain sparse. In this paper, we distill the essence of the covariate shift problem and focus on estimating the average $\E_{\widetilde{\xx}\sim p_{\mathrm{test}}}\ff(\widetilde{\xx})$, of any unknown and bounded function $\ff$,   given labeled training samples $(\xx_i, \ff(\xx_i))$, and unlabeled test samples $\widetilde{\xx}_i$; this is a core subroutine for several widely studied learning problems. We give several efficient algorithms, with provable sample complexity and computational guarantees. Moreover, we provide the first rigorous analysis of algorithms in this space when $\ff$ is unrestricted, laying the groundwork for developing a solid theoretical foundation for covariate shift problems.
\end{abstract}

\section{Introduction}

A common assumption in several machine learning algorithms is that training and test data come from the same distribution, which does not hold in many real-world applications. {\em Covariate shift}, first introduced by~\textcite{shimodaira2000improving}, is a widely used assumption for addressing such distribution shifts. It assumes that the conditional distribution of labels given the features remains the same across training and test data, i.e., $\ptr(\yy|\xx) = \pte(\yy|\xx)$, while the marginal distributions of the features differ, i.e., $\ptr(\xx) \neq \pte(\xx)$. Covariate shift has applications across various domains, including Natural Language Processing~\cite{yamada2010semi}, Signal Processing~\cite{yamada2012no}, Brain-computer interfacing~\cite{li2010application}, Spam Filtering~\cite{bickel2006dirichlet}, Human activity recognition~\cite{hachiya2012importance}, speaker identification~\cite{yamada2010semi}, and biomedical engineering~\cite{li2010application}.

In this paper, we capture the core of covariate shift and introduce the following general problem, which we refer to as \emph{covariate-shifted mean estimation}. Let $\ff: \bR^d \to [-1,1]$ be a bounded function unknown to the algorithm. Given $\epsilon>0$,  along with $n_{tr}$ labeled samples $(\xx_i,\ff(\xx_i))$ for $\xx_i\sim \ptr$ and $n_{te}$ unlabeled samples $\widetilde{\xx_i}\sim\pte$, the goal is to compute $Z$ such that $|Z-\E_{\xx\in \pte}\ff(\xx)|\leq \epsilon$. The question is: for a target accuracy $\varepsilon$, how many samples $n_{tr}, n_{te}$ are sufficient to achieve this guarantee? Notably, if both $\ptr$ and $\pte$ were known exactly, solving this problem would be significantly easier. In this case \begin{equation*}
    \E_{\xx\sim \ptr} \frac{\pte(\xx)}{\ptr(\xx)}\ff(\xx) = \E_{\xx\sim \pte}\ff(\xx),
\end{equation*}
providing direct access to an unbiased estimator of the desired quantity. Moreover, if the ratio $\pte(\xx)/\ptr(\xx)$ is bounded for a typical $\xx \sim \ptr$, we can control the variance of this estimator\footnote{A variant of this condition is necessary. For instance, if $\pte$ and $\ptr$ have disjoint support, estimating $\E_{\xx\in \pte}\ff(\xx)$ is impossible, even if we know both densities $\pte$ and $\ptr$ exactly since we have no information about $\ff$ on the support of $\pte$.}, and therefore estimate $\E_{\xx \sim \pte} \ff(\xx)$ from a finite sample drawn from $\ptr$. As a result, most of the interest is focused on approximately learning the density ratio $\pte(\xx)/\ptr(\xx)$, also called {\it importance weights}. Independently, there is a rich body of work on importance weight estimation in both the theoretical and applied literature, for example, ~\cite{CYM10,wen2015correcting,hachiya2012importance,gopalan2022multicalibrated,sugiyama2012density}. However, surprisingly, relatively few studies attempt to address the following two fundamental questions systematically.
\begin{enumerate}
    \item What approximation guarantees can be obtained for a specific algorithm attempting to estimate the density ratio with a given number of samples?
    \item What kind of approximation guarantees do we need for the estimated density ratio to serve as a good proxy of the true density ratio for covariate shift?
\end{enumerate}

Since the covariate-shifted mean estimation problem cannot be efficiently solved in full generality, certain regularity conditions must be imposed~--- either on the unknown function $\ff$ or on the unknown densities $\ptr, \pte$. Most prior work has focused on the development of algorithms under specific restrictions on $\ff$, 
with the most significant theoretical results based on \emph{kernel mean matching}. In this work, we primarily focus on the latter approach, where we assume  constraints on the densities instead.

\subsection{Related Works}
Before presenting our results, we summarize the current state of theoretical results in this space.

\paragraph{Kernel Mean Matching based Approaches}
The most popular algorithm for the covariate shift problem is the celebrated \emph{Kernel Mean Matching} (KMM), introduced by \cite{huang2006correcting}~---~one of the few results in this space with end-to-end theoretical guarantees. We discuss it in more detail here, as it shares a superficial resemblance to our most general result due to its foundation in \emph{Reproducing Kernel Hilbert Space} (RKHS) theory.

Consider an RKHS $\cH$ with a reproducing kernel map $\Phi$ (see~\Cref{sec:rkhs-prelim} for a brief overview of RKHS notation),  where the kernel is bounded such that $\|\Phi_{\xx}\|_{\cH} \leq 1$ for all $\xx$. The standard analysis of KMM assumes that $\ff \in \cH$ is bounded in the $\cH$ norm, i.e., $\|\ff\|_{\cH} \leq M$, and the density ratio satisfies $\pte(\xx)/\ptr(\xx) \leq B$. Theorem~1 of~\textcite{yu2012analysis} shows that, under these conditions, collecting $n_{tr} \gtrsim (MB/\varepsilon)^2$ samples from $\ptr$ and $n_{te} \gtrsim (M/\varepsilon)^2$ samples from $\pte$ suffices to estimate $\E_{\tilde{\xx} \sim \pte} \ff(\tilde{\xx})$ within an error of~$\varepsilon$. More detailed description of the algorithm can be found in \Cref{sec:KMMComparison}.

We further observe that if $\ff$ belongs to a function class $\cF$ that allows it to be learned with arbitrarily small error~--- i.e., one can find $\hat{\ff}$ such that, $\E_{\xx \sim \ptr} |\ff(\xx) - \hat{\ff}(\xx)| \leq \varepsilon'$, with respect to $\ptr$~--- there is a naive approach to solve the covariate-shifted mean estimation problem. Specifically, one can first approximate $\ff$ with $\hat{\ff}$ up to an error $\varepsilon' \ll \varepsilon$, and use~$\hat{\ff}$ in place of $\ff$ when estimating $\E_{\tilde{\xx}\sim\pte} \ff(\tilde{\xx})$. 

In this case, we obtain:
\begin{equation*}
\left|\E_{\tilde{\xx} \sim \pte} \hat{\ff}(\tilde{\xx}) - \E_{\tilde{\xx}\sim\pte} \ff(\tilde{\xx})\right| \leq \E_{\tilde{\xx} \sim \pte} |\hat{\ff}(\tilde{\xx}) - \ff(\tilde{\xx})| \leq B \E_{\xx \sim \ptr} |\hat{\ff}(\xx) - \ff(\xx)| \leq B \varepsilon'.
\end{equation*}
Thus, setting $\varepsilon':= \varepsilon/B$ ensures the desired accuracy. The estimate $\hat{\ff}$ can be found using $L_1$ kernel regression with $n_{tr} \approx M^2/{\varepsilon'}^2 = (MB/\varepsilon)^2$ samples. Then $n_{te} \approx M^2/\varepsilon^2$ samples suffice to estimate $\E_{\tilde{\xx} \sim 
\pte} \hat{\ff}(\xx)$ within an error of~$\varepsilon$. See~\Cref{sec:KMMComparison} for more details.

This highlights that the standard analysis of the KMM algorithm in a simple realizable scenario falls short of providing an asymptotic improvement in sample complexity over the naive plug-in method. Therefore, a more fine-grained analysis is necessary to justify the widely observed superior empirical performance of KMM.

\paragraph{Importance Weights Estimation via Classification}

A common approach to estimating the density ratio is to reduce it to a classification problem. Specifically, consider a distribution over pairs $(\xx, y)$ where $y \in \{ 0, 1\}$ is uniform, 
and $\xx$ is drawn from $\ptr$ when $y=0$, and from $\pte$ when $y=1$. The Bayes-optimal classifier for this problem outputs $\Pr[y=1 | \xx]$, which can be transformed into the density ratio $\pte(\xx)/\ptr(\xx)$.

Given the extensive literature on classification, one can train a model to approximate $\Pr[y=1 | \xx]$ and then apply the same transformation to obtain an estimate of the density ratio. This estimated ratio can, in turn, serve as importance weights for solving the covariate-shifted mean estimation problem.

Logistic regression is a widely used method for estimating $\Pr[y=1 | \xx]$ in this setting, while the use of more powerful models, such as \emph{kernel logistic regression} has also been explored~\cite{bickel2009discriminative,sugiyama2012density}.  

Despite the popularity of this approach to handle the covariate shift, no end-to-end theoretical guarantees are known. In this work, we bridge this gap by providing sample complexity and approximation guarantees for covariate-shifted mean estimation using logistic regression and kernel logistic regression under specific model assumptions.

\paragraph{Separately Estimating Both Densities}

A widely held belief in importance weight estimation is that separately learning the densities $\ptr$ and $\pte$, and then using their ratio, $\widehat{\pte}/\widehat{\ptr}$, as an estimate of $\pte/\ptr$ is both inefficient~--- since density estimation in high-dimensional spaces often requires exponentially many samples~--- and, more importantly, insufficient (see, for example, Section 2.1 in \cite{huang2006correcting}, or Section 2.2 in \cite{yu2012analysis}).

This intuition is natural, as even a small error in estimating the denominator $\ptr$, can lead to a significant error in the density ratio. Consequently, it remains unclear whether, even in the situation where we could obtain estimates for $\ptr$ and $\pte$ within $\varepsilon$ \emph{total variation} (TV) distance in time polynomial in $1/\varepsilon$, this would be helpful in solving the covariate shift problem.

Surprisingly, we show that with only a polynomial increase in sample complexity, it is indeed possible to solve the covariate-shifted mean estimation problem for any pair of distributions within an efficiently learnable class -- provided that $\pte(\xx)/\ptr(\xx)$ remains reasonably bounded for most $\xx \sim \ptr$. There is a rich body of work in distribution learning theory demonstrating that many distribution families can be efficiently learned, such as mixtures of Gaussians~\cite{moitra2010settling,liu2022clustering}, low-dimensional log-concave distributions~\cite{diakonikolas2017learning}, or graphical models with bounded treewidth~\cite{narasimhan2004paclearning}.

\subsection{Our Results}

In this paper, we introduce several algorithms for the covariate-shifted mean estimation problem, providing theoretical guarantees on approximation bounds and sample complexity. Our first result is an algorithm with small sample complexity when both $\ptr$ and $\pte$ are multivariate $d$-dimensional Gaussian distributions. Surprisingly, this is the first work to establish formal bounds for covariate shift, even in the case of Gaussian distributions 

\begin{theorem}[Informal Statement of Theorem~\ref{thm:CVGauss}]\label{thm:GaussIntro}
    Let $\ptr$ and $\pte$ be $d$-dimensional Gaussian distributions which are close to each other.  There is an algorithm (Algorithm~\ref{alg:CVGauss}), such that for any $\ff$ satisfying $\sup_{\xx\in \mathbb{R}^d}|\ff(\xx)|\leq 1$,  requires at most $O\left(d^2/\varepsilon^2 \log\frac{1}{\delta}\right)$ samples of $(\xx_i,\ff(\xx_i)), \xx_i \sim \ptr$ and $\widetilde{\xx_i}\sim\pte$, and returns $Z$ such that with probability at least $1-\delta$,
\[
|Z-\E_{\xx\in \pte}\ff(\xx)|\leq \varepsilon.
\]
In the case when $\pte,\ptr$ are isotropic Gaussians, $O\left(d/\varepsilon^2 \log(1/\delta)\right)$ suffice (Algorithm~\ref{alg:CVGaussIso}).
\end{theorem}

Our algorithms indicate that it is sufficient to learn the training and test Gaussian distributions to within an $\varepsilon$ {\it total variation distance}. Given such estimates, $\widehat{\ptr}$ and $\widehat{\pte}$, the estimator $\frac{\widehat{\pte}}{\widehat{\pte}}\ff$ is good enough for our purposes. 

Next, we extend our approach to a broader class of probability distributions. Specifically, we show that if $\ptr$ and $\pte$ belong to any efficiently learnable class of distributions, the covariate-shifted mean estimation problem can be solved with polynomially many samples by learning both distributions individually.

\begin{theorem}[Informal Statement of~\Cref{thm:shift-TV}]\label{thm:TV-Intro}
Let $\ff$ be a bounded function, $\ptr, \pte$ be a pair of distributions, and let $B>0$, be such that $\Pr_{\xx\sim\pte}(\pte(\xx)/ \ptr(\xx) > B/4) \leq \varepsilon$. Let $\widehat{r}$ be s.t. $\widehat{r} = \widehat{\pte}/\widehat{\ptr}$ for some $\widehat{\pte}$, and $\widehat{\ptr}$ satisfying $d_{TV}(\pte, \widehat{\pte}) \leq \varepsilon$ and $d_{TV}(\ptr, \widehat{\ptr}) \leq \varepsilon / B$.

Then given access to $\hat{r}$ and using $O(B^2/\varepsilon^2\log(1/\delta))$ samples samples $(\xx_i, \ff(\xx_i))$, $\xx_i\sim \ptr$ we can return $Z$ (Algorithm~\ref{alg:CV}) such that, with probability at least $1-\delta$,
    \begin{equation*}
        |Z -\E_{\xx\sim \pte} \ff(\xx)| \leq O(\varepsilon).
    \end{equation*}
\end{theorem}

Our algorithm again employs the estimator $\frac{\widehat{\pte}}{\widehat{\ptr}}\ff$ where $\widehat{\pte},\widehat{\ptr}$, are approximations of $\pte$ and $\ptr$ respectively~--- except of dismissing all samples for which $\frac{\widehat{\pte}}{\widehat{\ptr}}$ exceeds a threshold $B$. This challenges the common belief that separately estimating both densities is insufficient for accurately estimating the density ratio. We prove that when the tail of the density ratio is bounded, independently estimating the training and test distributions to a small error is sufficient to solve the covariate-shifted mean estimation problem. While Algorithm~\ref{alg:CV} applies to a broad class of distributions, for Gaussian distributions, Algorithm~\ref{alg:CVGauss} still achieves better sample complexity (Theorem~\ref{thm:GaussIntro}).

Our remaining generalizations build on Algorithm~\ref{alg:CV} and Theorem~\ref{thm:TV-Intro}, focusing on efficiently computing the approximate ratio $\widehat{r}$. As previously mentioned, logistic regression is a popular method for estimating such quantities, yet no theoretical guarantees exist for its performance. We prove that when $\pte$ and $\ptr$ belong to the same exponential family, logistic regression successfully computes the required $\widehat{r}$. Thus, we provide the first theoretical analysis of logistic regression-based density estimation methods. Our analysis, combined with Theorem~\ref{thm:TV-Intro}, leads to the following result.

\begin{theorem}[Informal Statement of Theorem~\ref{thm:logistic-regression-result}]
      Assume that $\ptr$ and $\pte$ belong to the same exponential family parameterized by $\theta \in \bR^D$. If the distance between $\pte$ and $\ptr$ is at most $B$, we can solve the covariate-shifted mean estimation problem using at most $O(B^4 D \varepsilon^{-8})$ samples.
\end{theorem}

Finally, we present our most general result for the covariate-shifted mean estimation problem. We prove that if $\ln (\ptr/\pte)$ belongs to a \emph{Reproducing Kernel Hilbert Space}, the problem can be solved with polynomial sample complexity using an efficient algorithm. Specifically, we show that \emph{kernel logistic regression} can be used to obtain the required estimate $\widehat{r}$, which can then be used in Algorithm~\ref{alg:CV} to obtain the required solution. Notably, our result only assumes regularity of the density ratio $\ptr/\pte$ with either of these densities can potentially be pathological by itself.

\begin{theorem}[Informal Statement of Theorem~\ref{thm:covariate-shift-kernel-logistic-regression}]
    Let $\cH$ be a RKHS of functions over $U$, and $\ptr, \pte$ a pair of distributions over $U$, such that $\ln (\ptr(\xx) / \pte(\xx)) \in \cH$. Assume that the kernel $K$ associated with $\cH$ is bounded by $D$ and the distance between $\ptr$ and $\pte$ is at most $B$. We can then solve the covariate-shifted mean estimation problem using at most $O(B^4 D \varepsilon^{-8})$ samples.
\end{theorem}

 \section{Preliminaries}

\subsection*{Exponential Families}
\begin{definition}[Probability Density of Exp. Distributions]\label{def:exp}
    We define an {\it exponential family} of probability distributions as those
distributions whose density (relative to parameter $\boldsymbol{\theta} \in \Theta \subset \bR^D$) have the following general form:
\[
p(\xx|\boldsymbol{\theta}) = \hh(\xx)\cdot \exp\left(\boldsymbol{\theta}^{\top}\boldsymbol{T}(\xx)-A(\boldsymbol{\theta})\right),
\]
for a parameter $\boldsymbol{\theta}$, and functions $\hh$ and $\boldsymbol{T}$. The function $A$ is then automatically determined as the normalizing constant.
\end{definition}

This family captures several widely used probability distributions such as Gaussian, Gamma, Binomial, Bernoulli, and Poisson.

In particular the log density ratio for any two distributions in the exponential family, is an affine function of the reparmetrization $\boldsymbol{T}$
\begin{equation*}
    \ln(p(\xx | \boldsymbol{\theta}_1)/p(\xx | \boldsymbol{\theta}_2)) = \langle \boldsymbol{\theta}_1 - \boldsymbol{\theta}_2, \boldsymbol{T}(x)\rangle - A(\boldsymbol{\theta}_1) + A(\boldsymbol{\theta}_2). 
\end{equation*}
 \paragraph{Notation.} We use {\it high-probability bounds} to refer to sample complexity bounds that are true with probability at least $1-\delta$ with the dependence of $\delta$ on the sample complexity is $\log\frac{1}{\delta}$. We use $d_{TV}(p,q)$ to denote the total variation distance between distributions $p$ and $q$. The notation $\supp(p)$ is used to denote the support of the distribution $p$.
We define for $p \geq 1$, the discrepancy metric between two probability distributions as
\begin{equation*}
    R_{p}(p_1 || p_2) := \E_{\xx \sim p_2} \left(\frac{p_1(\xx)}{p_2(\xx)}\right)^p.
\end{equation*}

Note that this is closely related with the Renyi divergence between $p_1$ and $p_2$, that is $D_{p}(p_1 || p_2) = \frac{1}{p - 1} \exp(R_p(p_1 || p_2)$. The quantity $R_p$ will be more convenient for us.
 %\section{Extended preliminaries}

\begin{definition}[Distance Measure]\label{def:Q}
For any probability distributions $p,q,\widehat{p}$, define 
\begin{equation*}
    \Qdiv{1}{\hat{p}}{p}{q} := \E_{\xx\sim q} \left|\frac{p(\xx)}{\widehat{p}(\xx)} - 1 \right|,
\end{equation*}
and more generally
\begin{equation*}
    \Qdiv{s}{\widehat{p}}{p}{q} := \left(\E_{\xx\sim q} \left|\frac{p(\xx)}{\widehat{p}(\xx)} - 1\right|^s \right)^{1/s}.
\end{equation*}
\end{definition}
By standard norm comparison for any $1\leq t_1 \leq t_2$, $Q_{t_1} \leq Q_{t_2}$. 

\begin{definition}[KL Divergence]
For any two probability distributions $p$ and $q$, the KL-divergence, $D_{KL}(p \, || \, q)$, between $p$ and $q$ is defined as
\begin{equation*}
    D_{KL}(p \, || \, q) := \E_{\xx\sim p} \ln \frac{ p(\xx)}{q(\xx)}.
\end{equation*}
    
\end{definition}

\begin{definition}[TV Distance] 
For a pair $p, q$ of probability distributions that are absolutely continuous with respect to some common measure, define the total variation distance between them as,
\[
d_{TV}(p, q) := \int |p(\xx) - q(\xx)|\, \mathrm{d}\xx
\]
    
\end{definition}

\begin{theorem}[Pinsker's Inequality]
\label{thm:pinsker}
    The famous Pinsker inequality sates that for any pair $p, q$ of probability distributions that are absolutely continuous with respect to some common measure, $d_{TV}(p, q) \leq \sqrt{2 D_{KL}(p\, || \, q)}$.
\end{theorem}

\subsection*{Gaussian Random Variables}
The following are well-known results about learning $d$-dimensional Gaussian random variables. Let $\xx_1,\cdots,\xx_{2m}\sim \cN(\boldsymbol{\mu},\boldsymbol{\Sigma})$ and define,
    \begin{equation}\label{eq:samples}
          \widehat{\boldsymbol{\mu}} = \frac{1}{m}\sum_{i\in[m]}\xx_i, \quad \widehat{\boldsymbol{\Sigma}} = \frac{1}{2m}\sum_{i\in [m]}(\xx_{2i}-\xx_{2i-1})(\xx_{2i}-\xx_{2i-1})^{\top}.
    \end{equation}

\begin{lemma}[TV Learning Gaussians, Theorem C.1,~\cite{ashtiani2020near}]
    There exists an absolute constant $C$ such that if we take $2m = 2C(d^2+d\log(1/\delta))/\epsilon^2$ samples $\xx_1,\cdots,\xx_{2m}\sim \cN(\boldsymbol{\mu},\boldsymbol{\Sigma})$, then for $  \widehat{\boldsymbol{\mu}}, \widehat{\boldsymbol{\Sigma}}$ as defined in Eq.~\eqref{eq:samples}, with probability $1-\delta$,
\[
d_{TV}\left(\cN(\boldsymbol{\mu},\boldsymbol{\Sigma}),\cN(\widehat{\boldsymbol{\mu}},\widehat{\boldsymbol{\Sigma}})\right) \leq \epsilon/2.
    \]
\end{lemma}

\begin{lemma}[Estimating Mean of Gaussians, Lemma C.2~\cite{ashtiani2020near}]\label{lem:MeanGauss}
    If $m\geq (2d + 6\sqrt{d\log(2/\delta)})/\epsilon^2$, then we have for $\widehat{\boldsymbol{\mu}}$ as defined in Eq.~\eqref{eq:samples},
    \[
    \Pr\left[(\widehat{\boldsymbol{\mu}}-\boldsymbol{\mu})^{\top}\boldsymbol{\Sigma}^{-1}(\widehat{\boldsymbol{\mu}}-\boldsymbol{\mu}) \geq \frac{\epsilon^2}{2}\right] \leq \frac{\delta}{2}.
    \]
\end{lemma}

\begin{lemma}[Estimating the Covariance of Gaussians, Lemma C.3~\cite{ashtiani2020near}]\label{lem:CovGauss}
    There exists an absolute constant $C$ such that if $m \geq C(d^2+d\log(1/\delta))/\epsilon^2$, then for $  \widehat{\boldsymbol{\Sigma}}$ as defined in Eq.~\eqref{eq:samples} with probability at least $1-\delta/2$,
    \[
\left\|\boldsymbol{\Sigma}^{-1/2}\widehat{\boldsymbol{\Sigma}}\boldsymbol{\Sigma}^{-1/2} -\II\right\|_{op}\leq \frac{\epsilon}{\sqrt{2d}}.
    \]
\end{lemma}

Note that since for any positive semi-definite matrix $\AA$, it holds that $\|\AA\|_F\leq \|\AA\|_{op}\cdot \sqrt{d}$, the above guarantee is stronger than the standard on the Frobenius norms of the covariance matrix.

\subsection*{Subexponential Random Variables}

\begin{definition}[Subesponential Random Variables]
    We say that a random variable $Z$ over $\bR$ is $(\sigma^2, B)$-subexponential, if and only if for every $\lambda \leq 1/B$, $\E \exp(\lambda (Z - \E Z)) \leq \exp(\lambda^2\sigma^2/2)$.
\end{definition}
The following fact is easy to derive directly from the definition of subexponential random variables. 
\begin{fact}
    If $Z_1, \ldots Z_n$ are $(\sigma_i, B_i)$-subexponential independent random variables, then $Z = \sum_i Z_i$ is $(\sigma, B)$-subexponential, where $\sigma = \sqrt{\sum \sigma_i^2}$, and $B = \max_i(B_i)$.
\end{fact}
The following are well known facts about Gaussian random variables.
\begin{fact}\label{fact:quad-gauss}
    If $Z = \mathcal{N}(\mu, 1)$, then $Z^2$ is $(16, 1/4)$-subexponential. Further, if $Z$ is a Gaussian vector $Z \sim \cN(\boldsymbol{\mu}, \boldsymbol{\Sigma})$, with $\|\boldsymbol{\Sigma}\|_{op}\leq O(1)$, and $A$ is a symmetric matrix, then $Z^T A Z$ is $(O(\|A\|_F), O(\|A\|_{op}))$-subgaussian, where $\|A\|_{op}$ is the largest eigenvalue of $A$.
\end{fact}

\subsection*{Rademacher complexity}
\begin{definition}
Rademacher complexity of a family $\mathcal{F}$ of functions $\Theta \to \bR$, with respect to the distribution $\cD$ is defined as
\begin{equation*}
    \mathcal{R}_{n, \cD} (\mathcal{F}) := \E_{\xx_1, \ldots \xx_n \sim \cD} \E_{\sigma_1, \ldots \sigma_n} \frac{1}{n} \sup_{f \in \mathcal{F}} \sum_i \sigma_i f(\xx_i),
\end{equation*}
where $\sigma_1, \ldots \sigma_n$ are independent $\{\pm 1\}$ Rademacher random variables.
\end{definition}
We will repeatedly use that left-composition of the entire function class with a single univariate Lipschitz function does not increase the Rademacher complexity. The following is a corollary of Lemma 5 in \cite{meir2003generalization}, see also Theorem 4.12 in \cite{ledoux2013probability}.
\begin{theorem}
\label{thm:rademacher-lipschitz-composition}
If $\ell : \bR \to \bR$ is a 1-Lipschitz function,  $\mathcal{F}'$ is arbitrary family of functions from $\Theta$ to $\bR$, $\cD$ is a distribution on $\Theta$ and $\mathcal{F} := \{ \ell \circ f : f \in \mathcal{F}'\}$, then
\begin{equation*}
    \mathcal{R}_{n, \cD}(\cF) \leq \mathcal{R}_{n, \cD}(\cF').
\end{equation*}
\end{theorem}

The main reason Rademacher complexity is useful is to provide generalization bounds: using a finite sample, we can estimate simultanously expectation of all functions in the class with bounded Rademacher complexity.
\begin{theorem}
\label{thm:rademacher-generalization-bound}
    Let $\xx_1, \ldots \xx_n\sim \cD$ be a sequence of i.i.d. random variables, and let $\mathcal{F}$ be a family of functions. Then
    \begin{equation}
        \E_{\xx} \sup_{f \in \mathcal{F}} \left|\frac{1}{n} \sum_{i \leq n}f(\xx_i) - \E_{\xx \sim \cD} f(\xx)\right| \leq 2 \mathcal{R}_{n, \cD}(\cF). \label{eq:rademacher-gen-bound}
    \end{equation}
\end{theorem}
Note that if $\mathcal{F}$ is in addition a family of functions bounded by $c$, a stronger concentration guarantees are known for~\eqref{eq:rademacher-gen-bound}. That is, with probability $1-\delta$, we have
\begin{equation*}
    \sup_{f \in \mathcal{F}} \left|\frac{1}{n} \sum_{i \leq n}f(\xx_i) - \E_{\xx} f(\xx)\right| \lesssim \mathcal{R}_{n, \cD}(\cF) + \frac{c \sqrt{\log (1/\delta)}}{n},
\end{equation*}
but we will not be using this bound explicitly.

\subsection*{Concentration of Random Variables}

\begin{lemma}[Hoeffding's Inequality]\label{lem:hoeff}
    Let $X_1,\cdots, X_n$ be independent random variables such that $a\leq X_i\leq b$, and $\E[X_i] = \mu$ for all $i$.
Consider the average of these random variables, $\overline{X} = \frac{1}{n}\sum_i X_i$. Then, for all $t>0$,
\[
\Pr\left[|\overline{X} - \mu|\geq t\right] \leq 2\exp\left\{-\frac{2nt^2}{(b-a)^2}\right\}.
\]
\end{lemma}

\section{Covariate Shift for Gaussian Distributions \label{sec:gaussians-optimal}}

In this section, we present our results for Gaussian distributions. We prove that it is sufficient to learn the means and variances of the training and test distributions to within an $\varepsilon$ error in order to solve the covariate-shifted mean estimation problem with $\varepsilon$ accuracy. In particular, we prove,
\begin{theorem}\label{thm:CVGauss}
Let $\ptr = \cN(\mtr,\str)$ and $\pte = \cN(\mte,\ste)$, where $\|\mtr-\mte\|_2\leq O(1)$ and $\|\str^{-1}-\ste^{-1}\|_F \leq O(1)$. Also, let $\|\mtr\|_2\leq O(1), \|\str^{-1}\|_{op}\leq O(1)$. For any $\ff$ such that $\sup_{\xx\in \mathbb{R}^d}|\ff(\xx)|\leq 1$, Algorithm~\ref{alg:CVGauss} returns $Z$ such that with probability at least $1-\delta$,
\begin{equation}\label{eq:ResGauss}
    |Z-\E_{\xx\in \pte}\ff(\xx)|\leq \varepsilon.
\end{equation}

Furthermore, Algorithm~\ref{alg:CVGauss} requires sampling at most $O\left(\frac{d^2}{\varepsilon^2 }\log\frac{1}{\delta}\right)$ samples of $(\xx_i,\ff(\xx_i)), \xx_i \sim \ptr$ and $\widetilde{\xx_i}\sim\pte$. 
\end{theorem}

For isotropic Gaussians, we can prove a tighter bound on the sample complexity.
\begin{theorem}\label{thm:CVGaussIso}
Let $\ptr = \cN(\mtr,\II)$ and $\pte = \cN(\mte,\II)$, where $\|\mtr-\mte\|_2\leq O(1)$. For any $\ff$ such that $\sup_{\xx\in \mathbb{R}^d}|\ff(\xx)|\leq 1$, Algorithm~\ref{alg:CVGaussIso} returns $Z$ such that with probability at least $1-\delta$,
\begin{equation}\label{eq:ResGaussIso}
    |Z-\E_{\xx\in \pte}\ff(\xx)|\leq \varepsilon.
\end{equation}

Furthermore, Algorithm~\ref{alg:CVGaussIso} requires sampling at most $O\left(\frac{d}{\varepsilon^2 }\log\frac{1}{\delta}\right)$ samples of $(\xx_i,\ff(\xx_i)), \xx_i \sim \ptr$ and $\widetilde{\xx_i}\sim\pte$.
\end{theorem}

\begin{algorithm}
\caption{Covariate shift for General Gaussian $\pte$ and $\ptr$}
 \label{alg:CVGauss}
  \begin{algorithmic}[1]
\STATE {\bfseries Input:} $d,\varepsilon,\delta$
 \STATE $K = O\left(\frac{d^2}{\varepsilon^2}\log \frac{1}{\delta}\right)$
\STATE Generate $K$ samples $\xx_i$ from $\ptr$ and $K$ samples $\widetilde{\xx}_i$ from $\pte$.
 \STATE Compute $\widehat{\mtr} = \frac{2}{K}\sum_{i\in [K/2]} \xx_i$
 \STATE Compute $\widehat{\mte} = \frac{2}{K}\sum_{i\in [K/2]} \widetilde{\xx}_i$
 \STATE Compute $\widehat{\str} = \frac{1}{K}\sum_{i\in [K/2]}(\xx_{2i}-\xx_{2i-1})(\xx_{2i}-\xx_{2i-1})^{\top}$
 \STATE Compute $\widehat{\ste} = \frac{1}{K}\sum_{i\in [K/2]}(\widetilde{\xx}_{2i}-\widetilde{\xx}_{2i-1})(\widetilde{\xx}_{2i}-\widetilde{\xx}_{2i-1})^{\top}$
\STATE $\widehat{\ptr} = \cN(\widehat{\mtr},\widehat{\str})$, $\widehat{\pte} = \cN(\widehat{\mte},\widehat{\ste})$
\FOR{$t = 1,\cdots,O(\log(1/\delta))$}
\STATE Generate $m= O\left(\varepsilon^{-2}\right)$ samples $(\yy_i,\ff(\yy_i)),$ with $\yy_i \sim \ptr.$
 \FOR{$i\in [m]$}
 \STATE $Z_i = \frac{\widehat{\pte}(\yy_i)}{\widehat{\ptr}(\yy_i)}\ff(\yy_i)$
 \ENDFOR
 \STATE $\overline{Z}_t = \frac{1}{m}\sum_{i=1}^m Z_i$
 \ENDFOR
 \STATE  {\bfseries Return:} Median of $\overline{Z}_t$'s
 \end{algorithmic}
 \end{algorithm}

 \begin{algorithm}
\caption{Covariate shift for Isotropic Gaussian $\pte$ and $\ptr$}
 \label{alg:CVGaussIso}
  \begin{algorithmic}[1]
\STATE {\bfseries Input:} $d,\varepsilon,\delta$
 \STATE $K = O\left(\frac{d}{\varepsilon^2}\log\frac{1}{\delta}\right)$
\STATE Generate $K$ samples $\xx_i$ from $\ptr$ and $K$ samples $\widetilde{\xx}_i$ from $\pte$.
 \STATE Compute $\widehat{\mtr} = \frac{2}{K}\sum_{i\in [K/2]} \xx_i$
\STATE Compute $\widehat{\mte} = \frac{2}{K}\sum_{i\in [K/2]} \widetilde{\xx}_i$
\STATE $\widehat{\ptr} = \cN(\widehat{\mtr},\II)$, $\widehat{\pte} = \cN(\widehat{\mte},\II)$
\FOR{$t = 1,\cdots,O(\log(1/\delta))$}
\STATE Generate $m= O\left(\varepsilon^{-2}\right)$ samples $(\yy_i,\ff(\yy_i)),$ with $\yy_i \sim \ptr.$
 \FOR{$i\in [m]$}
 \STATE $Z_i = \frac{\widehat{\pte}(\yy_i)}{\widehat{\ptr}(\yy_i)}\ff(\yy_i)$
 \ENDFOR
 \STATE $\overline{Z}_t = \frac{1}{m}\sum_{i=1}^m Z_i$
 \ENDFOR
 \STATE  {\bfseries Return:} Median of $\overline{Z}_t$'s
 \end{algorithmic}
 \end{algorithm}

 Our analysis proves that the estimator $X = \frac{\widehat{\pte}}{\widehat{\ptr}}\ff$ is a good estimator. Note that this estimator is biased. We divide the analysis into two parts: in the first part, we show that our estimator has a small bias, and in the second part, we provide a bound on its variance. Finally, we apply concentration bounds to establish the total sample complexity. 

 For the first part, we prove the following result. The proofs of these lemmas are deferred to Appendix~\ref{sec:ProofsGauss}.

 \begin{restatable}[Bound for Isotropic Gaussian Distributions]{lemma}{IsoGaussBias}
\label{lem:isotropic-gaussians-result}
Let $\ptr = \cN(\mtr,\II)$ and $\pte = \cN(\mte,\II)$ such that $\|\mtr\|_2\leq O(1)$ and $\|\mtr-\mte\|_2\leq B\leq O(1)$. Further, let $\widehat{\ptr}$ and $\widehat{\pte}$ be as defined in Algorithm~\ref{alg:CVGauss}. Then for any function $\ff$ such that $\sum_{\xx\in \mathbb{R}^d}|\ff(\xx)|\leq 1$,
\[
\left|\E_{\xx\sim \ptr} \frac{\widehat{\pte}(\xx)}{\widehat{\ptr}(\xx)} \ff(\xx) - \E_{\xx \sim \pte} \ff(\xx)\right| \leq O(\varepsilon).
\]
\end{restatable}

\begin{restatable}[Bound for General Gaussian Distributions]{lemma}{GaussBias}
\label{lem:Q_gauss_non_iso}
    Let $\pte := \cN(\mte, \ste), \ptr := \cN(\mtr, \str)$ and $\widehat{\ptr} = \cN(\widehat{\mtr} , \widehat{\str}),\widehat{\pte} = \cN(\widehat{\mte} , \widehat{\ste})$ be as defined in Algorithm~\ref{alg:CVGauss}. If $\|\mtr-\mte\|_2 \leq B\leq O(1)$, $\|\str^{-1}\|_{op}\leq O(1)$ then, for any function $\ff$, such that $\sup_{\xx\in \mathbb{R}^d}|\ff(\xx)|\leq 1$,
    \[
\left|\E_{\xx\sim \ptr} \frac{\widehat{\pte}(\xx)}{\widehat{\ptr}(\xx)} \ff(\xx) - \E_{\xx \sim \pte} \ff(\xx)\right| \leq O(\varepsilon).
\]
\end{restatable}

In the next part, we show that the variance of our estimator is small. The proof is again deferred to Appendix~\ref{sec:ProofsGauss}.

\begin{restatable}[Variance of our Estimator]{lemma}{VarianceGauss}\label{lem:Var}
     If $\|\str^{-1}\|_{op}, \|\ste^{-1}\|_{op} \leq O(1)$, $\|\str^{-1}-\ste^{-1}\|_F\leq O(1)$, and $\|\mtr\|_2, \|\mtr-\mte\|_2\leq O(1)$, then, the variance of our estimator can be bounded as,
    \[
    \E_{\xx\sim \ptr} \left[\left(\frac{\widehat{\pte}(\xx)}{\widehat{\ptr}(\xx)}\ff(\xx) - \E_{\xx\sim\pte}\ff(\xx)\right)^2\right] \leq O(1).
    \]
\end{restatable}

Having established that our estimator has a small bias, i.e., $O(\epsilon)$, and a small variance, $O(1)$, the result follows directly from standard Chebyshev's inequality. 

For the high-probability statement, we apply the standard technique: given an estimator $Z$ that lies within the desired interval $I$ with probability $2/3$, we can take the median of $O(\log(1/\delta))$ independent realizations of this estimator to reduce the failure probability to $\delta$, using the Chernoff bound.

\subsection*{Proof of Theorems~\ref{thm:CVGauss} and~\ref{thm:CVGaussIso}}
\begin{proof}[Proof of Theorem~\ref{thm:CVGauss}]

From Lemmas~\ref{lem:MeanGauss} and \ref{lem:CovGauss}, it is sufficient to use $O(d^2/\varepsilon^2 \log1/\delta)$ samples to obtain $\widehat{\ptr}$ and $\widehat{\pte}$ to the required accuracy.

Now, from Lemma~\ref{lem:Q_gauss_non_iso} our estimator $X = \frac{\widehat{\pte}(\xx)}{\widehat{\ptr}(\xx)}\ff(\xx)$ for $\xx \sim \ptr$ satisfies 
\begin{equation}\label{eq:bias}
    |\E_{\xx\sim \ptr} X - \E_{\xx\sim \pte}\ff(\xx)| \leq \varepsilon.
\end{equation}

Further, from Lemma~\ref{lem:Var}, 
\[
\Var(X) = \E_{\xx\sim \ptr} \left[\left(\frac{\widehat{\pte}(\xx)}{\widehat{\ptr}(\xx)}\ff(\xx) - \E_{\xx\sim\pte}\ff(\xx)\right)^2\right] \leq O(1).
\]

We now use Chebyshev's inequality to obtain our sample complexity bounds. Since the variance of the estimator is at most $O(1)$ the variance of $Z =\frac{1}{m}\sum_i \frac{\widehat{\pte}}{\widehat{\ptr}}\ff(\xx_i)$ is at most $O(1/m)$. Now from Chebyshev's inequality and Eq.~\eqref{eq:bias}, for $m\geq O(1/\epsilon^2)$,
\[
\Pr\left[\left|\frac{1}{m}\sum_i \frac{\widehat{\pte}}{\widehat{\ptr}}\ff(\xx_i)-\E_{\xx\sim \pte}\ff(\xx)\right|\geq \epsilon\right] \leq \frac{Var(Z)}{\epsilon^2}\leq 0.1.
\]

We can now use the standard median trick to boost the probability, i.e., repeat the sampling process of sampling $m$ points $O(\log (1/\delta))$ times. Let $\overline{X}$ denote the average of the $m$ samples above. Now, the median trick says generate $\overline{X}_i$ for $i =1,\cdots, O(\log(1/\delta))$. The median $Y$ of these $\overline{X}_i$'s satisfies with probability at least $1-\delta$,
\[
|Y - \E_{\xx\sim\pte}\ff(\xx)|\leq \epsilon
\]
\end{proof}

\begin{proof}[Proof of Theorem~\ref{thm:CVGaussIso}]
From Lemma~\ref{lem:MeanGauss} it is sufficient to use $O(d/\epsilon^2\log 1/\delta)$ samples to obtain $\widehat{\ptr}$ and $\widehat{\pte}$ to the required accuracy.

Now, from Lemma~\ref{lem:isotropic-gaussians-result} our estimator $X = \frac{\widehat{\pte}(\xx)}{\widehat{\ptr}(\xx)}\ff(\xx)$ for $\xx \sim \ptr$ satisfies 
\begin{equation*}
    |\E_{\xx\sim \ptr} X - \E_{\xx\sim \pte}\ff(\xx)| \leq \varepsilon.
\end{equation*}
The remaining proof follows the same as in the proof of Theorem~\ref{thm:CVGauss}.
\end{proof}

\section{TV Learnability implies Covariate Shift}\label{sec:TV}

In this section, we extend our algorithms to more general distributions. We show that, under mild conditions on the training and test distributions, it is sufficient to learn both distributions to a small total variation distance. In particular, we prove the following:

\begin{theorem}
\label{thm:shift-TV}
    Let $\ff$ be a bounded function, $\ptr, \pte$ be a pair of distributions, such that for $B>0$, $\Pr_{\xx \sim \pte}(\pte(\xx) / \ptr(\xx) > B/4) \leq \varepsilon$. Let $\widehat{r}$, be s.t. $\widehat{r} = \widehat{\pte}/\widehat{\ptr}$ for some $\widehat{\pte}$, and $\widehat{\ptr}$ satisfying $d_{TV}(\pte, \widehat{\pte}) \leq \varepsilon$ and $d_{TV}(\ptr, \widehat{\ptr}) \leq \varepsilon / B$. Then Algorithm~\ref{alg:CV} returns $Z$ such that, with probability at least $0.9$, we get
    \begin{equation*}
        |Z -\E_{\xx\sim \pte} \ff(\xx)| \leq O(\varepsilon).
    \end{equation*}
    Furthermore, Algorithm~\ref{alg:CV} uses only $O(B^2 / \varepsilon^2)$ samples $(\xx_i, \ff(\xx_i))$, $\xx_i\sim \ptr$.
\end{theorem}

 \begin{algorithm}
 \caption{Covariate shift for any $\pte$ and $\ptr$ with bounded tails}
 \label{alg:CV}
 \begin{algorithmic}[1]
 \STATE {\bfseries Input:} $\widehat{r} =\widehat{\pte}/\widehat{\ptr},B,\varepsilon$
 \STATE $K = O(B^2/\varepsilon^2)$
 \STATE Generate $K$ samples $(\xx_i,\ff(\xx_i))$ from $\ptr$.
 \IF{$\widehat{r}(\xx_i)\leq B$} 
 \STATE $Z_i = \widehat{r}(\xx_i)\ff(\xx_i)$
\ELSE
\STATE $Z_i = 0$
 \ENDIF
 \STATE {\bfseries Return:} $Z = \frac{1}{K}\sum_{i=1}^K Z_i$
 \end{algorithmic}
 \end{algorithm}
In the next section, we will show how to compute the density ratio $\widehat{r}$ used in Algorithm~\ref{alg:CV}, satisfying the conditions of the above theorem, for certain classes of probability distributions, without explicitly estimating $\ptr$ and $\pte$.

\begin{remark}
    \label{rem:lower-bound}
    This sample complexity upper bound should be contrasted with the following simple lower bound: let $B$ be such that $\Pr_{\xx\sim \pte}(\pte(\xx)/ \ptr(\xx)> B) = 2 \varepsilon$. Even with exact knowledge of $\pte$ and $\ptr$, we require $\Omega(B/\varepsilon)$ samples to solve the covariate-shifted mean estimation problem up to error $\varepsilon$. A simple proof of this lower bound can be found in~\Cref{sec:ProofsTV}.
\end{remark}
A relatively common assumption in the covariate shift literature, which is stronger than the one we use, is that the density ratio $\pte(\xx)/\ptr(\xx)$ is bounded everywhere by $B$. We chose to relax this assumption because it is not even satisfied for two non-equal univariate Gaussian distributions.

To put the sample complexity in~\Cref{thm:shift-TV} into context, note that we could apply it when $\ptr$ and $\pte$ are isotropic Gaussian distributions, as in~\Cref{thm:CVGaussIso}, with $|\mu_1 - \mu_2| \leq O(1)$. In this case, for the desired accuracy $\varepsilon$, we would need to set $B := (1/\varepsilon)^{1 + o(1)}$ and learn an isotropic Gaussian $\ptr$ up to an error of $(1/\varepsilon)^{2 + o(1)}$ in total variation distance. This can be done using $O(d/\varepsilon^{4 + o(1)})$ samples. While the dependency on the dimension $d$ matches that of~\Cref{thm:CVGauss}, the dependency on the accuracy $\varepsilon$ is polynomially worse.

The following technical lemma will be crucial in our proof.

\begin{restatable}{lemma}{ApproximationBiasTV}
\label{lem:approximation_2}
Let $\ptr, \pte$ be a pair of distributions, and $S \subseteq U$ be such that $\Pr_{\xx\sim \pte}(S) \geq 1-\varepsilon$ and $\Pr_{\xx\sim \ptr}(S) \geq 1-\varepsilon$. Let $\widehat{\ptr}$ and $\widehat{\pte}$ be a pair of estimations of $\ptr,\pte$ respectively with $\supp (\widehat{\pte}) \subseteq \supp (\widehat{\ptr}) \subseteq S$. Moreover, assume that $\forall \xx\in S$, $\widehat{\pte}(\xx) / \widehat{\ptr}(\xx) \leq B$. Then
\begin{equation*}
    \left|\E_{\xx\sim \ptr} \frac{\widehat{\pte}(\xx)}{\widehat{\ptr}(\xx)} \ff(\xx) - \E_{\xx\sim \pte} \ff(\xx)\right|  \leq B d_{TV}(\ptr, \widehat{\ptr}) + d_{TV}(\pte, \widehat{\pte}).
\end{equation*}
\end{restatable}
The proof of the above is in Appendix~\ref{sec:ProofsTV}. We now prove the main result of the section.
\begin{proof}[Proof of \Cref{thm:shift-TV}]
    Let $S := \{ \xx : \widehat{\pte}(\xx) / \widehat{\ptr}(\xx) \leq B \}$, and $S' := \{\xx : \pte(\xx) / \ptr(\xx) \leq B/4\}$. Moreover, define sets $A_{\mathrm{te}} := \{ \xx : \widehat{\pte}(\xx) \leq 2 \pte(\xx) \}$, and $A_{\mathrm{tr}} := \{ \xx : \ptr(\xx) \leq 2 \ptr(\xx)\}$.

    We would first prove that $\pte(S) \geq 1 - O(\varepsilon)$ and $\ptr(S) \geq 1 - O(\varepsilon)$. Note that since $\frac{\widehat{\pte}}{\widehat{\ptr}} =  \frac{\pte}{\ptr}\cdot \frac{\widehat{\pte}}{\pte} \cdot \frac{\ptr}{\widehat{\ptr}}$, $A_{\mathrm{te}} \cap A_{\mathrm{tr}} \cap S' \subseteq S$. Let $U$ denote the entire space. We have
    \begin{itemize}
        \item $\pte(U \setminus A_{\mathrm{te}}) \leq \widehat{\pte}(U \setminus A_{\mathrm{te}}) + \varepsilon \leq O(\varepsilon)$,
        \item $\pte(S' \setminus A_{\mathrm{tr}}) \leq \frac{B}{4} \cdot \ptr(S' \setminus A_{\mathrm{tr}}) \leq \frac{B}{4}\cdot \ptr(U \setminus A_{\mathrm{tr}}) \leq O(\varepsilon)$, 
        \item $\pte(S') \leq \varepsilon$ by assumption.
    \end{itemize}
    This yields that $\pte(S) \geq \pte(A_{\mathrm{te}} \cap A_{\mathrm{tr}} \cap S') \geq 1 - O(\varepsilon)$ and therefore $\widehat{\pte}(S) \geq 1 - O(\varepsilon)$. Furthermore, from the definition of $S$ we also have $\widehat{\ptr}(S^c) < \widehat{\pte}(S^c)/B \leq(\pte(S^c) + \varepsilon)/B \leq O(\varepsilon/B)$. Therefore, we also have, ${\ptr}(S)\geq 1-O(\varepsilon)$.

Now, consider the conditional distributions $\widehat{\pte}|_S$ and $\widehat{\ptr}|_S$, which are defined as, $p|_S(\xx) = p(\xx)$, if $\xx\in S$ and $0$ if $\xx\notin S$. We get the desired bound on the density ratio on the set $S$.
    \begin{equation*}
        \frac{\widehat{\pte}|_S(\xx)}{\widehat{\ptr}|_S(\xx)} \leq  B 
    \end{equation*}

    Note that conditioning on $S$ doesn't introduce too much error, i.e., if $d_{TV}(\pte, \widehat{\pte}) \leq \varepsilon$ and $\pte(S) \geq 1-\varepsilon$ then $d_{TV}(\pte, \widehat{\pte}|_S) \leq O(\varepsilon)$ by triangle inequality. Similarly, $d_{TV}(\widehat{\ptr}|_S, \ptr) \leq O(\varepsilon/B)$ and we can apply the lemma~\Cref{lem:approximation_2} to deduce that,
    the bias is at most $\varepsilon$, i.e.,
     \[
        \left|\E_{\xx\sim \ptr} \frac{\widehat{\pte}(\xx)}{\widehat{\ptr}(\xx)} \ff(\xx) - \E_{\xx\sim \pte} \ff(\xx)\right| \leq O(\varepsilon).
\]
    
    Finally, the random variable $Z :=\ff(\xx)\mathbf{1}_S(\xx) \frac{\widehat{\pte}|_S(\xx)}{\widehat{\ptr}|_S(\xx)}$ for $\xx \sim \pte$ is always bounded by $B$, hence $\Var(Z) \leq B$. The final bound now follows directly from the application of the Chebyshev inequality. Further, we can boost the probability using the standard median trick.
    
This concludes the proof of Theorem~\ref{thm:shift-TV}.
\end{proof}

\section{Learning Density Ratio by Classification}\label{sec:logistic}

In this section, we present an algorithm based on logistic regression that approximates the density ratio $\widehat{r}$ such that $\widehat{r} = \widehat{\pte}/\widehat{\ptr}$ for some pair of distributions, with the guarantees that $d_{TV}(\widehat{\ptr}, \ptr) \leq \varepsilon$ and $d_{TV}(\widehat{\pte}, \pte) \leq \varepsilon$. This holds when both $\pte$ and $\ptr$ belong to the same exponential family (Definition~\ref{def:exp}), i.e., have the same values of the functions $\hh$ and $\boldsymbol{T}$. Moreover, the probabilities $\widehat{\ptr}$ and $\widehat{\pte}$ provide guarantees for the density ratio, i.e., $\widehat{r} = \frac{\widehat{\pte}}{\widehat{\ptr}}$ is a good approximation to $r = \frac{\pte}{\ptr}$, making it suitable for use as importance weights in the covariate shift problem (after appropriate truncation). Although logistic regression has been applied in practice for estimating density ratios in several previous works, such theoretical guarantees remained unknown. We prove the following:

\begin{theorem}
\label{thm:logistic-regression-result}
    Assume that $\ptr$ and $\pte$ belong to the same exponential family, with $\pte(\xx) = p(\xx | \tte)$ and $\ptr(\xx) = p(\xx | \ttr)$. Let $\str = \E_{\xx \sim \ptr} \boldsymbol{T}(\xx) \boldsymbol{T}(\xx)^{T}$ and similarly for $\ste$. Define $B_{1} := \Tr (\str + \ste)$ and assume, in addition, that $\|\ttr - \tte\|_2 + |A(\ttr) - A(\tte)| \leq B_{2}$.
    
    As long as $R_2(\ptr || \pte) \leq B_3$, we can solve the covariate-shifted mean eastimation problem using $O(B_1 B_2^2 B_3^4 \varepsilon^{-8})$ samples.
\end{theorem}

We now present the high-level idea of our approach. The details and proof of \Cref{thm:logistic-regression-result} is in Appendix~\ref{sec:ProofsLogistic}.

\paragraph{General scheme.}
Consider a distribution $\mathcal{D}_0$ over pairs $(\xx,y)$ where $\xx \in U$ and $y \in \{ \pm 1 \}$, s.t. $\Pr(y = 1) = 1/2$ and $\Pr(\xx | y=1) = \pte(\xx)$ and $\Pr(\xx | y=-1) = \ptr(\xx)$. Then by Bayes rule
\begin{align*}
    \Pr(y = -1 | \xx) &= \frac{\Pr(\xx | y=-1)}{\Pr(\xx|y=1) + \Pr(\xx|y=-1)} \\
    & = \frac{\ptr(\xx)}{\pte(\xx) + \ptr(\xx)} \\
    &= \frac{1}{1 + \pte(\xx)/\ptr(\xx)}.
\end{align*}

Hence, if we can find a Bayes optimal classifier for the distribution $(\xx,y)$ calculating $\Pr(y=-1 | \xx)$, we could invert it to get a density ratio:
\begin{equation}
\label{eq:classification-dr}
    \frac{\pte(\xx)}{\ptr(\xx)} = \frac{1}{\Pr(y=-1 | \xx)} - 1.
\end{equation}

A known method (see for instance \cite{BS06}) of estimating the ratio $\pte(\xx)/\ptr(\xx)$ is then training a classifier for the binary classification problem (to predict $y$ given features $\xx$ from distribution $\mathcal{D}_0$), and using its output in place of $\Pr(y=-1 | \xx)$ in \Cref{eq:classification-dr}.

We show that if we use the negative log-likelihood as a loss function for classification training (also known as the cross-entropy loss), to any classifier $\hat{r}(\xx)$ we have found, we can associate a pair of distributions $\widehat{\pte}$ and $\widehat{\ptr}$, s.t. $\widehat{\pte}/\widehat{\ptr} = \frac{1}{\hat{r}(x)} - 1$, and the regret of the classifier $\hat{r}$ (difference between the population loss of the classifier, and the population loss of the Bayes optimal classifier $\Pr_{\cD_0}[y=-1 | x]$) is an upper bound for the sum of KL-divergences $D_{KL}(\pte || \widehat{\pte}) + D_{KL}(\ptr || \widehat{\ptr})$; see~\Cref{lem:population-regret-is-kl}.

In particular, whenever we can guarantee that this regret is low, we know that the estimated density ratio corresponds to the density ratio of a pair of distributions $\widehat{\ptr}, \widehat{\pte}$ that are both KL-close to the actual distributions $\ptr, \pte$. Using the Pinskers inequality, this can be translated into TV-distance guarantees, and via~\Cref{thm:shift-TV}, after appropriate truncation, such a density ratio can be used as importance weight for the covariate shift problem.

Interestingly, this method never explicitly recovers the densities $\widehat{\pte}$ and $\widehat{\ptr}$ which are shown to exist~--- it only recovers the density ratio.

\paragraph{Using logistic regression.}

When the distributions $\pte$ and $\ptr$ are from the same exponential family, using logistic regression to find the classifier provides the right density ratio estimate. In particular, for parameters $\tte$ and $\ttr$, and functions $\hh,\boldsymbol{T}$, $\pte = \hh(\xx) \cdot \exp\left(\tte^{\top}\boldsymbol{T}(\xx)-A(\tte)\right)$, and $\ptr = \hh(\xx) \cdot \exp\left(\ttr^{\top}\boldsymbol{T}(\xx)-A(\ttr)\right)$.

Now, for these distributions, 
\begin{align*}
    & \Pr_{(\xx,y) \sim \mathcal{D}_0}[y=-1 | \xx]  = \frac{1}{1 + \pte(\xx)/\ptr(\xx)}\\
    & = \frac{1}{1 + \exp\left(\boldsymbol{T}(\xx)^{\top}(\tte-\ttr) + A(\tte)-A(\ttr)\right)}
\end{align*}
The logistic model is given by a joint distribution $(\xx, y)$, parameterized by $\boldsymbol{\theta}^* \in \bR^d$ and $s^* \in \bR$, where
\begin{equation}
    \Pr[y = -1 | \xx] = \frac{1}{1 + \exp(\langle \boldsymbol{\theta}^*, \xx\rangle + s^*)}.
    \label{eq:logistic}
\end{equation}
Hence, the distribution $\mathcal{D}_1$ of $(\boldsymbol{T}(\xx), y)$ for $(\xx, y) \sim \mathcal{D}_0$ indeed follows the logistic model, Eq.\eqref{eq:logistic} with parameters $\ttheta^* = \tte - \ttr$ and $s^* = A(\tte) - A(\ttr)$.
Note that for the distribution following the logistic model we have
\begin{align*}
    \Pr[y=1 | \xx] &= \frac{1}{1 + \exp(-\langle \boldsymbol{\theta}, \xx \rangle - s)},
\end{align*}
and more succinctly
\begin{align*}
    \Pr[y | \xx] & = (1 + \exp(-y(\langle \boldsymbol{\theta}, \xx \rangle +s)))^{-1}.
\end{align*}

Since the empirical minimization of negative log-likelihood for logistic regression is a convex problem, we can efficiently find a minimizer on a finite sample. Using a Rademacher complexity bound for the associated class of loss functions, the minimizer on a finite sample will be an approximate minimizer with respect to the population loss. Concretely we can find $(\hat{\ttheta}, \hat{s})$ with population loss $\varepsilon$ close to optimal, i.e., regret bounded by $\varepsilon$, where the loss is the negative log-likelihood. This regret bound, by~\Cref{lem:population-regret-is-kl}, lets us control the total KL-divergence $D_{KL}(\pte || \widehat{\pte}) + D_{KL}(\ptr || \widehat{\ptr})$. By Pinsker inequality, this implies TV-closeness between the respective distributions and our approximations. We can then invoke~\Cref{thm:shift-TV}, to conclude that after truncation, the density ratio $\widehat{\pte}/\widehat{\ptr}$ can be used as importance weights for the covariate-shifted mean estimation problem.

\paragraph{Error bounds on the estimated parameter $\hat{\theta}$ in the Gaussian case.}

In the case where both distributions are $d$-dimensional Gaussian vectors with known covariance, the sample complexity bound obtained by~\Cref{thm:logistic-regression-result} is easily seen to be suboptimal. For instance, when $\ptr = \cN(\mtr, \II)$, $\pte = \cN(\mte, \II)$, s.t. $\|\mtr - \mte\| \leq O(1)$ according to~\Cref{thm:logistic-regression-result}, we can solve the covariate-shifted mean estimation problem using $O(d/\varepsilon^8)$ samples (in this case we would have $B_1 \approx d, B_2 \leq O(1)$, and $B_3 \leq O(1)$). In contrast~\Cref{thm:CVGaussIso} shows that by learning the means of both Gaussian distributions separately, we can solve the covariate shift problem using $O(d/\varepsilon^{2})$ samples.

As it turns out, specifically in the case of Gaussians, we can improve the sample complexity in~\Cref{thm:logistic-regression-result}, and we can recover the importance weights $\hat{r}$, using only $O(\max\{d \log d, d/\varepsilon^{2}\})$ samples using logistic regression, nearly matching the guarantees of~\Cref{thm:CVGaussIso}.

This is a consequence of new results for the error bounds on the \emph{estimated parameter} $\|\hat{\ttheta} - \ttheta^*\|$ in logistic regression, as opposed to classical bounds showing only that the (population) loss at $\hat{\ttheta}$ is not much larger than the minimum value attained at $\ttheta^*$. Bounds of this form are only known when covariate $\xx$ itself is a sub-Gaussian random variable, so this statement is much less general then~\Cref{thm:logistic-regression-result}, but provides much stronger results in this scenario.  We present these improvements in Appendix~\ref{sec:LogisticGauss}

\section{Reproducing Kernel Hilbert Spaces~---~Kernel Trick for Density Estimation}\label{sec:RKHS}
We now present a generalization of logistic regression beyond exponential families.

Specifically, we consider the case where $\ln (\ptr/\pte)$ lies in a \emph{Reproducing Kernel Hilbert Space} $\mathcal{H}$. We show that using \emph{kernel logistic regression}~\cite{zhu2005kernel}, we can find an approximation to the density ratio $\ptr/\ptr$, which is good enough to solve the covariate shift problem. 

Before stating the main theorem, let us briefly review the theory behind RKHS. For a more detailed exposition, refer to~\cite{ghojogh2021reproducing}.

\subsection{Quick RKHS Preliminary \label{sec:rkhs-prelim}}

Consider a Hilbert space $\cH$ of functions over $U$ (not necessarily finitely dimensional), together with an inner product $\langle f,g \rangle_\cH$. If for each $x \in U$ the evaluation functional $\mathrm{ev}_x(f) := f(x)$ is bounded (i.e. for each $x \in U$, there is a bound $T_{x}$, such that for each $f \in \cH$ we have $f(x) \leq T_x \|f\|_{\cH}$), then we say that $(\cH, \langle \cdot, \cdot\rangle_{\cH})$ is a \emph{Reproducing Kernel Hilbert Space (RKHS)}. 

Using standard functional-analytic tools, this implies the existence of a \emph{feature map} $\Phi: U \to \cH$ --- i.e., with any point $x \in U$, we can associate a function $\Phi_x \in \cH$ satisfying, for every $f \in \cH$, $\langle f, \Phi_x \rangle_{\cH} = f(x)$. We can further associate with this RKHS a kernel $K : U\times U \to \mathbb{R}$, given as $K(x_1, x_2) := \langle \Phi_{x_1}, \Phi_{x_2}\rangle_{\cH}$. As it turns out, every such kernel is positive semidefinite, and in fact, this is a complete characterization of RKHS --- every positive semidefinite kernel $K: U \times U \to \mathbb{R}$ corresponds to some RKHS on $U$. For a given kernel $K$, the corresponding feature map is $\Phi_x := K(x, \cdot)$, and the subspace of finite sums
$\sum_i \lambda_i K(x_i, \cdot)$ is dense in $\cH$.

\paragraph{Why are RKHS good --- the Kernel trick}
The key tool related to RKHS is the so-called 'kernel trick' and, more specifically, the relatively simple yet powerful Representer Theorem. This theorem states that the solution to a general optimization problem over the (potentially infinitely dimensional) $\cH$ can be found by reducing it to a finite-dimensional problem. For example,  consider a sample $(x_1, y_1), (x_2, y_2), \ldots (x_n, y_n)$ where $x_i \in \mathcal{U}$ and $y_i \in \mathbb{R}$. 
\begin{theorem}[Representer Theorem\cite{ghojogh2021reproducing}]
\label{thm:representer}
    Let $\cH$ be an RKHS of functions over $U$, together with the feature map $\Phi: U \to \cH$. Let $E : (U \times \bR \times \bR)^n \to \bR$ be an arbitrary error map,  $(x_1, y_1), \ldots (x_n, y_n) \in U \times \bR$ be a sample, and $B \geq 0$ a non-negative bound. Then the minimizer $\hat{\theta}$ of
    \begin{equation}
        \min_{\boldsymbol{\theta} \in \cH ~ : ~ \|\boldsymbol{\theta}\|_{\cH} \leq B} E((x_1, y_1, \boldsymbol{\theta}(x_1)), \ldots (x_n, y_n, \boldsymbol{\theta}(y_n))
    \end{equation}
    is a linear combination
    \begin{equation}
        \hat{\boldsymbol{\theta}} = \sum_i \gamma_i \Phi_{x_i}.
    \end{equation}
\end{theorem}
\paragraph{Kernel Logistic Regression}
Most of the content in this section is well-known. However, for completeness, we provide a detailed analysis of kernel logistic regression. Kernel logistic regression assumes that we observe samples $(\xx_i, y_i)$ where $\xx_i \in U$ is drawn from some distribution and $y_i|x_i \in \{\pm 1\}$ follows the logistic distribution,
\begin{equation*}
    \Pr[y| \xx] = (1 + \exp(y \boldsymbol{\theta}^*(\xx)))^{-1},
\end{equation*}
where $\boldsymbol{\theta}^* \in \cH$ for some RKHS $\cH$. Let us call the joint distribution of $(\xx,y)$ drawn this way as $\mathcal{D}_{\boldsymbol{\theta}^*}$. Note that as long as $\|\boldsymbol{\theta}^*\|_{\cH} \leq B$,  $\ttheta^*$ is a unique minimizer of the population loss function given by negative log-likelihood,
\begin{equation*}
\mathcal{L}_{\boldsymbol{\theta}}(\cD_{\ttheta^*}) := \E_{(x,y) \sim \cD_{\boldsymbol{\theta}^*}} \log(1 + \exp(y \boldsymbol{\theta}(\xx)).
\end{equation*}
That is $\ttheta_* = \text{arg min}_{\|\ttheta\|_{\cH} \leq B} \mathcal{L}_{\ttheta}(\cD_{\ttheta^*})$. 
Using a standard argument based on Rademacher complexity, we bound the error between the population loss and the empirical loss for all $\|\ttheta\|_{\cH} \leq B$. Combined with the fact that empirical loss minimization for kernel logistic regression can be efficiently solved (after applying the representer theorem, it reduces to an $n$-dimensional convex problem). We show the following.
\begin{lemma}
\label{lem:kernel-logistic-regression-regret-bound}
    Using $n \approx \E_{\xx}[K(\xx, \xx)] B^2/\varepsilon^2$ samples from $\cD_{\boldsymbol{\theta}^*}$ we can efficiently find $\hat{\boldsymbol{\theta}}$ satisfying regret bound
    \begin{equation*}
        \cL_{\hat{\boldsymbol{\theta}}}(\cD_{\ttheta^*})  - \cL_{\boldsymbol{\theta}^*}(\cD_{\ttheta^*}) \leq \varepsilon.
    \end{equation*}
\end{lemma}

See~\Cref{sec:kernel-logistic-regression-regret-bound} for a full proof.

\subsection{Covariate Shift via Kernel Logistic Regression}
We now give the main result of this section. If the log of the density ratio of the training and test distribution lies in some RKHS $\cH$, we can efficiently solve the covariate shift problem for these distributions. This result is obtained by combining the regret bounds from~\Cref{lem:kernel-logistic-regression-regret-bound}, which imply the KL-divergence bounds by~\Cref{lem:population-regret-is-kl}, with the Pinsker inequality and~\Cref{thm:shift-TV}.

\begin{theorem}
\label{thm:covariate-shift-kernel-logistic-regression}
    Let $\cH$ be a RKHS of functions over $U$, and $\ptr, \pte$ a pair of distributions over $U$, such that $\boldsymbol{\theta}^*(x) := \ln (\ptr(x) / \pte) \in \cH$.

    Assume that the kernel $K$ associated with $\cH$ satisfies $\sqrt{\E_{\xx \sim \ptr} K(x,x)} + \sqrt{\E_{\xx \sim \pte} K(\xx, \xx)} \leq D$ for all $x \in U$,  $\|\boldsymbol{\theta}^*\| \leq B_1$, and $R_2(\pte || \ptr) \leq B_2.$

    Then, we can solve the covariate-shifted mean estimation problem from $\ptr$ to $\pte$ using at most $O(D B_1^2 B_2^4 / \varepsilon^8)$ samples.
\end{theorem}
The full proof can be found in~\Cref{sec:covariate-shift-kernel-logistic-regression}.

Again, similar to the argument that using \Cref{thm:covariate-shift-kernel-logistic-regression} directly in the special case of univariate Gaussian distributions gives a suboptimal sample complexity, i.e., when $\ptr = \cN(\mu_{tr}, 1), \pte = \cN(\mu_{te}, 1)$  s.t. $\|\mu_1 - \mu_2\| \leq O(1)$, this result also gives a suboptimal sample complexity of $O(\varepsilon^{-8})$ instead of $O(\varepsilon^{-2})$.

\printbibliography

\newpage

\appendix

\section{Proofs from Section~\ref{sec:gaussians-optimal}}\label{sec:ProofsGauss}

 For the analysis of the algorithm, we let $\widehat{\pte} = \cN(\widehat{\mte},\widehat{\ste}),\widehat{\ptr} = \cN(\widehat{\mtr},\widehat{\str})$ to denote approximations to the $\pte,\ptr$ respectively.  We now begin with the first part of the analysis.

\subsection*{Bound on Bias of the Estimator: Proof of Lemmas~\ref{lem:isotropic-gaussians-result} and~\ref{lem:Q_gauss_non_iso}}
\begin{lemma}
\label{cor:approx-Q}
For any $\ptr,\pte$ with corresponding approximations $\widehat{\ptr}, \widehat{\pte}$, and bounded $\ff(\xx)$ such that $\sup_{\xx}|\ff(\xx)|\leq 1$ we have
\begin{equation*}
    \left|\E_{\xx\sim \ptr} \frac{\widehat{\pte}(\xx)}{\widehat{\ptr}(\xx)} \ff(\xx) - \E_{\xx \sim \pte} \ff(\xx)\right|  \leq \Qdiv{1}{\widehat{\ptr}}{\ptr}{\widehat{\pte}} + d_{TV}(\widehat{\pte}, \pte).
\end{equation*}
\end{lemma}
\begin{proof}
First by triangle inequality we can decompose

\begin{align}\label{eq:BasicBoundGauss}
    \left|\E_{\xx\sim \ptr} \frac{\widehat{\pte}(\xx)}{\widehat{\ptr}(\xx)} \ff(\xx) - \E_{\xx \sim \pte} \ff(\xx)\right| 
    \leq &  \left|\E_{\xx\sim \ptr} \frac{\widehat{\pte}(\xx)}{\widehat{\ptr}(\xx)} \ff(\xx) - \E_{\xx \sim \widehat{\pte}} \ff(\xx)\right| \notag \\
    & + |\E_{\xx \sim \widehat{\pte}} \ff(\xx) - \E_{\xx \sim {\pte}} \ff(\xx)|
\end{align}
We now consider the two terms separately. Using the fact that $|\ff(\xx)|\leq 1$, the second term gives,
\begin{align*}
|\E_{\xx \sim \widehat{\pte}} \ff(\xx) - \E_{\xx \sim {\pte}} \ff(\xx)| &= \left|\int_{\xx} \ff(\xx) (\widehat{\pte}(\xx) -\pte(\xx)) d\xx \right|\\
& \leq \sup_{\xx}|\ff(\xx)| \int_{\xx} |\widehat{\pte} -\pte| d\xx\\
& \leq d_{TV}(\widehat{\pte},\pte).
\end{align*}

We now consider the first term.

\begin{align*}
           \left|\E_{\xx\sim \ptr} \frac{\widehat{\pte}(\xx)}{\widehat{\ptr}(\xx)} \ff(\xx) - \E_{\xx \sim \widehat{\pte}} \ff(\xx)\right| & = \left|\E_{\xx\sim \ptr} \frac{\widehat{\pte}(\xx)\ptr(\xx)}{\ptr(\xx)\widehat{\ptr}(\xx)} \ff(\xx) - \E_{\xx \sim \widehat{\pte}} \ff(\xx)\right| \\
            &  = \left|\E_{\xx\sim \widehat{\pte}} \frac{\ptr(\xx)}{\widehat{\ptr}(\xx)} \ff(\xx) - \E_{\xx \sim \widehat{\pte}} \ff(\xx)\right| \\
            &  \leq \E_{\xx\sim \widehat{\pte}} \left|\frac{\ptr(\xx)}{\widehat{\ptr}(\xx)} - 1 \right|.
    \end{align*}
    Using these bounds in Eq.~\eqref{eq:BasicBoundGauss} gives the result.

\end{proof}

From the above lemma, we see that we need to first give a bound on $\Qdiv{1}{\widehat{\ptr}}{\ptr}{\widehat{\pte}}$. Since from Definition~\ref{def:Q}, we know that $\Qdiv{1}{\cdot}{\cdot}{\cdot}\leq \Qdiv{2}{\cdot}{\cdot}{\cdot}$, in the remaining part we will show how to bound $\Qdiv{2}{\widehat{\ptr}}{\ptr}{\widehat{\pte}}$.

\begin{lemma}
    \label{lem:q_log_subexp}Let  $p,\widehat{p},$ and $q$ denote probability distributions and
    let $Z$ be a random variable defined as $Z = \ln (p(\xx)/\widehat{p}(\xx))$ for $\xx\sim q$. If $\mu := \E Z$ and $Z$ is $(\sigma, 1/2)$-subexponential for $\sigma \leq O(1)$, then $\Qdiv{2}{\widehat{p}}{p}{q} \leq O(\mu + \sigma)$.
\end{lemma}
\begin{proof}
    We have
    \begin{align*}
        \Qdiv{2}{\widehat{p}}{p}{q}^2 & = \E_{\xx\sim q} \left( \frac{p(\xx)}{\widehat{p}(\xx)} - 1\right)^2 \\
        & = \E_{\xx\sim q} (\exp(Z) - 1)^2 = \E_{\xx\sim q} [\exp(2Z) - 2\exp(Z) +1 ].
    \end{align*}
    Since $Z$ is subexponential, we have $\E_{\xx\sim q}\exp(Z-\mu) \leq \exp( \sigma)$, and therefore, $\E_{\xx\sim q} \exp(2Z) \leq \exp(2 \mu + 2\sigma)$. Now, we will use that for any $y\geq 0$, $e^y \leq 1+y+O(1)y^2$ for $y\leq O(1)$ and $e^y \geq 1+y$ which yields
    \begin{align*}
        \E_{\xx\in q}[\exp(2Z) - 2\exp(Z) + 1] & = 1 + 2 \mu + O(1)\cdot(\mu+\sigma)^2 - 2(1 + \mu) - 2 \\
        & = O(1)\cdot (\mu + \sigma)^2.
    \end{align*}
\end{proof}

We begin with the easy case and first give a bound on the bias of our estimator in the case when $\ptr$ and $\pte$ are isotropic Gaussian distributions.

\IsoGaussBias*

\begin{proof}
    We will use \cref{lem:q_log_subexp} to show $\Qdiv{2}{\widehat{\ptr}}{\ptr}{\widehat{\pte}} \leq O(\varepsilon)$ in this case. The remaining follows from \cref{cor:approx-Q}.

    Let $p = \ptr$, $\widehat{p} = \widehat{\ptr}$, and $q = \widehat{\pte}$. We need to show that $Z$, as defined in \cref{lem:q_log_subexp} is $(\varepsilon,1/2)$ subexponential and $\E_{\xx\sim q}Z \leq O(\varepsilon)$.

    Now, $Z = \ln \ptr - \ln \widehat{\ptr} = \langle \mtr - \widehat{\mtr}, \xx\rangle + \frac{1}{2}\langle \widehat{\mtr} - \mtr, \mtr + \widehat{\mtr}\rangle$ is a Gaussian with mean $ \langle \widehat{\mtr} - \mtr,\widehat{\mte}\rangle + \frac{1}{2}\langle \widehat{\mtr} - \mtr, \mtr + \widehat{\mtr}\rangle$ and variance $\|\mtr - \widehat{\mtr}\|^2_2$. From Lemma~\ref{lem:MeanGauss}, we know that that $\|\mtr - \widehat{\mtr}\|^2_2\leq \varepsilon^2$, and therefore, $\Var(Z)\leq \varepsilon^2$. Further, the mean satisfies,  
 \[
    \E Z =  \langle \widehat{\mtr} - \mtr,2\widehat{\mte} + \mtr + \widehat{\mtr}\rangle 
    \leq\| \widehat{\mtr} - \mtr\|_2 \|2\widehat{\mte} + \mtr + \widehat{\mtr}\|_2.
\]
In the above, since  $\|\mtr-\mte\|_2\leq O(1)$, and $\|\mtr\|_2\leq O(1)$, we must have $\|\widehat{\mte}\|_2\leq O(1)$. Therefore,
 \[
    \E Z 
    \leq\| \widehat{\mtr} - \mtr\|_2 \left(2\|\widehat{\mte}\|_2 + \|\mtr\|_2 + \|\widehat{\mtr}\|_2\right) \leq O(\varepsilon).
\]

We now claim that $Z$ is $(\varepsilon,1/2)$ subexponential with mean bounded by $O(\varepsilon)$. This follows since $Z$ is gaussian with std deviation $\varepsilon$, the MGF of $Z$ is bouneded by $e^{\varepsilon^2 \lambda^2}$ for all $\lambda>0$, which implies $Z$ is $(\varepsilon,1/2)$ subexponential. This concludes the proof for isotropic Gaussian distributions.
\end{proof}

\GaussBias*

\begin{proof}
    We will deduce this as a corollary from \Cref{lem:q_log_subexp}. To this end, let $Z := \ln \ptr(\xx) - \ln \widehat{\ptr}(\xx)$ for $\xx\sim \widehat{\pte}$. Note that $Z = \xx^{\top} \AA \xx + \bb^T x + c$, where 
    \begin{align*}
        \AA & = (\str^{-1} - \widehat{\str}^{-1})/2 \\
        \bb & =\str^{-1} \mtr - \widehat{\str}^{-1} \widehat{\mtr} \\
        c & = (\mtr^{\top} \str^{-1} \mtr - \widehat{\mtr}^{\top} \widehat{\str}^{-1} \widehat{\mtr})/2.
    \end{align*}
 Now, from Fact~\ref{fact:quad-gauss}, $\xx^{\top} \AA \xx$ is $(\|\AA\|_F, \|\AA\|_{op})$-subexponential, and we know that $\|\AA\|_{op} \leq \|\AA\|_F\leq \sqrt{d}\|\AA\|_{op}$. 
  \[
    \|\widehat{\str}^{-1}-\str^{-1}\|_{op} \leq \|\str^{-1/2}\|_{op}\|\str^{1/2}\widehat{\str}^{-1}\str^{1/2}-\II\|_{op}\|\str^{-1/2}\|_{op}.
    \]
    From Lemma~\ref{lem:CovGauss}, we know that $\|\str^{-1/2}\widehat{\str}\str^{-1/2}-\II\|_{op} \leq \frac{\varepsilon}{2\sqrt{d}}$. This implies that all eigenvalues of $\str^{-1/2}\widehat{\str}\str^{-1/2}$ are between $1-\varepsilon/2\sqrt{d}$ and $1+\varepsilon/2\sqrt{d}$ and as a result, $\|\str^{1/2}\widehat{\str}^{-1}\str^{1/2}-\II\|_{op} \leq \varepsilon/\sqrt{d}$. Now, from assumption, $\|\str^{-1}\|_{op} \leq O(1)$, and we get that,
    \[
    \|\widehat{\str}^{-1}-\str^{-1}\|_{F}\leq \sqrt{d} \cdot \|\widehat{\str}^{-1}-\str^{-1}\|_{op}\leq O\left(\varepsilon\right).
    \]

 Therefore, $Z$ is $(O(\varepsilon), O(\varepsilon))$-subexponential. The random variable $\bb^{\top} \xx$ is $\|\bb\|_2$-subgaussian, and we can bound,
    \begin{equation*}
        \|\bb\|_2 = \|\str^{-1} \mtr - \widehat{\str}^{-1} \widehat{\mtr}\|_2 \leq \|\str^{-1}(\mtr-\widehat{\mtr})\|_2  +  \|(\widehat{\str}^{-1}-\str^{-1})\widehat{\mtr}\|_2  \leq O(\varepsilon).
    \end{equation*}
For the first term we use Lemma~\ref{lem:MeanGauss} and for the second term we use the previous calculation along with the assumption that $\|\mtr\|_2\leq O(1).$
    The sum of $\varepsilon$-subgaussian and $(\varepsilon, \varepsilon)$-subexponential random variable is $(O(\varepsilon), O(\varepsilon))$-subexponential again.

    Finally $\E Z = \Tr \AA + \bb^{\top}\mte + c$, and since $\|\mtr-\mte\|_2\leq O(1)$, and $\|\mtr\|_2\leq O(1)$ imply that $\|\mte\|_2\leq O(1)$, $\Tr \AA \leq \|\AA\|_F$, all we need to do is bound $c$.
    \begin{align*}
        (\mtr^{\top} \str^{-1} \mtr - \widehat{\mtr}^{\top} \widehat{\str}^{-1} \widehat{\mtr}) & \leq
        \|\mtr^{\top} (\str^{-1} - \widehat{\str}^{-1}) \mtr\|_2 + \langle \str^{-1}(\mtr - \widehat{\mtr}), (\mtr+ \widehat{\mtr}) \rangle \leq O(\varepsilon) .
    \end{align*}
    The final statement of the lemma now follows from Lemma~\ref{cor:approx-Q}.
\end{proof}

\subsection*{Bound on Variance of the Estimator: Proof of Lemma~\ref{lem:Var}}
In this section, we give a bound on the variance of our estimator $ \frac{\widehat{\pte}(\xx)}{\widehat{\ptr}(\xx)}\ff(\xx)$ for $\xx \sim \ptr$.

\begin{lemma}\label{lem:BoundSquaredVariance}
Let $X$ denote the estimator $\frac{\widehat{\pte}(\xx)}{\widehat{\ptr}(\xx)}\ff(\xx)$for $\xx \sim \ptr$. If $\|\str^{-1}\|_{op}\leq O(1)$, $\|\str^{-1}-\ste^{-1}\|_F\leq O(1)$, and $\|\mtr\|_2, \|\mtr-\mte\|_2\leq O(1)$, then
\begin{equation*}
    \E_{\xx \sim \ptr}[X^2] \leq 2 + 2\Qdiv{2}{\widehat{\ptr}}{\widehat{\pte}}{\ptr} \leq O(1).
\end{equation*}
\end{lemma}
\begin{proof}
\begin{align*}
    \E_{\xx \sim \ptr}[X^2]& = \E_{\xx \sim \ptr}\left[\left(\frac{\widehat{\pte}(\xx)}{\widehat{\ptr}(\xx)}\ff(\xx) - \ff(\xx) + \ff(\xx)\right)^2\right]\\
    & \leq 2\E_{\xx \sim \ptr}\left[\left(\frac{\widehat{\pte}(\xx)}{\widehat{\ptr}(\xx)}-1\right)^2\ff(\xx)^2\right]  + 2\E_{\xx\sim \ptr}[\ff(\xx)^2]\\
    & \leq 2\cdot \Qdiv{2}{\widehat{\ptr}}{\widehat{\pte}}{\ptr} + 2.
\end{align*}
In the last line we used $|\ff(\xx)|\leq 1.$

Now, we consider the distribution of $Z := \ln \widehat{\pte}(\xx) - \ln \widehat{\ptr}(\xx)$ for $\xx\sim \ptr$. Note that $Z = \xx^{\top} \AA \xx + \bb^T x + c$, where 
    \begin{align*}
        \AA & = (\widehat{\ste}^{-1}-\widehat{\str}^{-1})/2 \\
        \bb & =\widehat{\ste}^{-1} \widehat{\mte} - \widehat{\str}^{-1} \widehat{\mtr} \\
        c & = (\widehat{\mte}^{\top} \widehat{\ste}^{-1} \widehat{\mtr} - \widehat{\mtr}^{\top} \widehat{\str}^{-1} \widehat{\mtr})/2.
    \end{align*}
    
    Now, from Fact~\ref{fact:quad-gauss}, $\xx^{\top} \AA \xx$ is $(\|\AA\|_F, \|\AA\|_{op})$-subexponential, and since $\|\AA\|_{op} \leq \|\AA\|_F\leq \sqrt{d}\|\AA\|_{op}$ we will give a bound on $\|\AA\|_{F}$.

    From triangle inequality,
    \[
    \|\AA\|_{F}\leq \|\widehat{\ste}^{-1}-\ste^{-1}\|_{F}+ \|\ste^{-1}-\str^{-1}\|_{F} + \|\widehat{\str}^{-1}-\str^{-1}\|_{F}.
    \]
    Now,
    \[
    \|\widehat{\ste}^{-1}-\ste^{-1}\|_{op} \leq \|\ste^{-1/2}\|_{op}\|\ste^{1/2}\widehat{\ste}^{-1}\ste^{1/2}-\II\|_{op}\|\ste^{-1/2}\|_{op}.
    \]
    From Lemma~\ref{lem:CovGauss}, we know that $\|\ste^{-1/2}\widehat{\ste}\ste^{-1/2}-\II\|_{op} \leq \frac{\varepsilon}{2\sqrt{d}}$. This implies that all eigenvalues of $\ste^{-1/2}\widehat{\ste}\ste^{-1/2}$ are between $1-\varepsilon/2\sqrt{d}$ and $1+\varepsilon/2\sqrt{d}$ and as a result, $\|\ste^{1/2}\widehat{\ste}^{-1}\ste^{1/2}-\II\|_{op} \leq \varepsilon/\sqrt{d}$. Now, from assumption, $\|\ste^{-1}\|_{op} \leq O(1)$, and we get that,
    \[
    \|\widehat{\ste}^{-1}-\ste^{-1}\|_{F}\leq \sqrt{d} \cdot \|\widehat{\ste}^{-1}-\ste^{-1}\|_{op}\leq O\left(\varepsilon\right).
    \]
    Similarly, we can show that 
    \[
    \|\widehat{\str}^{-1}-\str^{-1}\|_{F}\leq O(\varepsilon).
    \]
    From assumption, we know that $\|\ste^{-1}-\str^{-1}\|_{F} \leq O(1)$, and as a result we get that $\|\AA\|_{op}\leq \|\AA\|_F \leq O(1).$

    Therefore, $Z$ is $(O(1), O(1))$-subexponential. The random variable $\bb^{\top} \xx$ is $\|\bb\|_2$-subgaussian, and we can bound,
    \begin{align*}
        \|\bb\|_2 & = \|\widehat{\ste}^{-1} \widehat{\mte} - \widehat{\str}^{-1} \widehat{\mtr}\|_2 \\
        &\leq \|\widehat{\ste}^{-1} \widehat{\mte} - \widehat{\ste}^{-1} \widehat{\mtr}\|_2+\|\widehat{\ste}^{-1} \widehat{\mtr} - \widehat{\str}^{-1} \widehat{\mtr}\|_2\\
        & \leq\|\widehat{\ste}^{-1}\|_{op}\|\widehat{\mtr} - \widehat{\mte}\|_2  + \|\widehat{\mtr}\|_2 \|\widehat{\str}^{-1} - \widehat{\ste}^{-1}\|_{op}\leq O(1).
    \end{align*}

    The sum of $O(1)$-subgaussian and $(O(1), O(1))$-subexponential random variable is $(O(1), O(1))$-subexponential again.

    Finally $\E Z = \Tr \AA + c$, and since $\Tr \AA \leq \|\AA\|_F$, all we need to do is bound $c$.
    \begin{align*}
        (\widehat{\mte}^{\top} \widehat{\ste}^{-1} \widehat{\mte} - \widehat{\mtr}^{\top} \widehat{\str}^{-1} \widehat{\mtr}) & \leq
        \|\widehat{\mtr}^{\top} (\widehat{\ste}^{-1} - \widehat{\str}^{-1}) \widehat{\mtr}\|_2 \\
        &+ \langle \widehat{\ste}^{-1}(\widehat{\mte} - \widehat{\mtr}), (\widehat{\mte}+ \widehat{\mtr}) \rangle \leq O(1) .
    \end{align*}
    From Lemma~\ref{lem:q_log_subexp}, we can now bound $\Qdiv{2}{\widehat{\ptr}}{\widehat{\pte}}{\ptr} \leq O(1).$
\end{proof}

\VarianceGauss*

\begin{proof}
    The proof follows from the following equation: for any random variable $X$ and $\mu$,
    \[
    \E[(X-\mu)^2] = \E[X^2] -2\mu\E[X] + \mu^2 \leq \E[X^2] + 2\mu|\E[X]-\mu|.
    \]
    Now, letting $X = \frac{\widehat{\pte}(\xx)}{\widehat{\ptr}(\xx)}\ff(\xx)$ and $\mu = \E_{\xx\sim\pte}\ff(\xx)$, from \Cref{lem:BoundSquaredVariance}, $\E[X^2] \leq O(1)$ and from \Cref{lem:Q_gauss_non_iso} $|\E[X]-\mu|\leq \varepsilon$. Furthermore, since $|\ff(\xx)|\leq 1$, we can further bound $\E_{\xx\sim\pte}\ff(\xx)\leq 1.$ Therefore $\E[(X-\mu)^2]\leq O(1)$.
\end{proof}

\section{Proofs from Section~\ref{sec:TV}}\label{sec:ProofsTV}

We will elaborate on~\Cref{rem:lower-bound}, showing that the dependence on $B$ in~\Cref{thm:shift-TV} is necessary.
\begin{lemma}
    Let $\ptr$ and $\pte$ be any pair of distributions s.t. $\Pr_{\xx \sim \pte}(\pte(\xx)/\ptr(\xx) \geq B) = 2\varepsilon$. Then, even if the distributions $\pte$ and $\ptr$ are known, we need $\Omega(B/\varepsilon)$ samples to solve the covariate shifted mean estimation problem for $\ptr$ and $\pte$.
\end{lemma}
\begin{proof}
Let $S = \{ \xx : \pte(\xx) / \ptr(\xx) > B \}$. By assumption, we have $\pte(S) = 2\varepsilon$, and moreover $\ptr(S) \leq 2 \varepsilon / B$ by the construction of the set $B$. The expectations with respect to $\pte$ of functions $\mathbf{1}_S$ and the constant $0$ function, differ by at least $2\varepsilon$, so in order to solve the covariate-shifted mean estimation problem, we need to be able to distinguish those two functions. But unless we sample $\Omega(B/\varepsilon)$ elements $\xx$ from $\ptr$ it is unlikely that any of those will be in $S$, and the algorithm will only have observed value $\ff(\xx) = 0$ in either of those two cases.
\end{proof}

\ApproximationBiasTV*

\begin{proof}
From \Cref{cor:approx-Q},
\[
\left|\E_{\xx\sim \ptr} \frac{\widehat{\pte}(\xx)}{\widehat{\ptr}(\xx)} \ff(\xx) - \E_{\xx\sim \pte} \ff(\xx)\right| \leq \Qdiv{1}{\widehat{\ptr}}{\ptr}{\widehat{\pte}} + d_{TV}({\pte},\widehat{\pte}).
\]
Therefore, it is sufficient to bound $\Qdiv{1}{\widehat{\ptr}}{\ptr}{\widehat{\pte}} \leq B d_{TV}(\widehat{\ptr}, \ptr)$. Since $\supp (\widehat{\pte}) \subseteq S$ and $\forall x\in S,\widehat{\pte}(\xx) / \widehat{\ptr}(\xx) \leq B$, we have
\begin{align*}
   \Qdiv{1}{\widehat{\ptr}}{\ptr}{\widehat{\pte}} &= \E_{\xx\sim \widehat{\pte}} \left| \frac{\ptr(\xx)}{\widehat{\ptr}(\xx)} - 1\right|\\
   &\leq B \E_{\xx\sim \widehat{\ptr}} \left| \frac{\ptr(\xx)}{\widehat{\ptr}(\xx)}- 1 \right|\\
   &= B d_{TV}({\ptr}, \widehat{\ptr}).
\end{align*}
\end{proof}

\section{Proofs from Section~\ref{sec:logistic}}\label{sec:ProofsLogistic}
\begin{lemma}
\label{lem:population-regret-is-kl}
Let $\cD_0$ be a distribution described above, $r^*(\xx) := \Pr_{\cD_0}[y=-1 | \xx]$ and consider a classifier $\hat{r} : U \to (0, 1)$, intended to estimate $r^*$.

Let $\cL_{\cD_0, \hat{r}}$ be an expected cross-entropy loss of classifier $\hat{r}$, namely
\begin{equation*}
    \cL_{\cD, \hat{r}} := - \E_{(\xx,y) \sim \cD} [ \mathbf{1}_{y=-1} \log \hat{r}(\xx) + \mathbf{1}_{y=1} \log (1-\hat{r}(\xx))].
\end{equation*}
Then there exist a pair of distributions $\widehat{\ptr}, \widehat{\pte}$ such that the population regret for a classifier $\hat{r} $ is equal
\begin{equation*}
    \mathcal{R}_{\cD, \hat{r}} := \mathcal{L}_{\cD, \hat{r}} - \mathcal{L}_{\cD, r^*} = \frac{1}{2} (D_{KL}(\pte || \widehat{\pte}) + D_{KL}(\ptr || \widehat{\ptr})) + (1 - H(y')),
\end{equation*}
where $H(y') \leq 1$ is a Shannon entropy of some binary random variable. In particular
\begin{equation*}
    D_{KL}(\pte || \widehat{\pte}) + D_{KL}(\ptr || \widehat{\ptr}) \leq 2 \mathcal{R}_{\hat{r}}.
\end{equation*}
\end{lemma}

\begin{proof}
Consider a joint distribution $\cD_1$ over $(\xx, y')$, s.t. the marginal distribution of $\xx$ is the same as in $\cD$, i.e. $\xx \sim \frac{1}{2} \ptr + \frac{1}{2} \pte$, and $\Pr[y' = -1 | \xx] := \hat{r}(\xx)$.

    We have
    \begin{align}
        \mathcal{L}_{\hat{r}} - \mathcal{L}_{r^*} & = \E_{(\xx,y) \sim \mathcal{D}_0} \mathbf{1}_{y=-1} (\log(\hat{r}(\xx)) - \log(r^*(\xx))) + \mathbf{1}_{y=1} ( \log(1 - \hat{r}(\xx)) - \log(1- r^*(\xx)))  \notag \\
        & = \frac{1}{2} \E_{x \sim \ptr} \log\frac{\hat{r}(\xx)}{r^*(\xx)} + \frac{1}{2} \E_{x\sim \pte} \log \frac{1-\hat{r}(\xx)}{1-r^*(\xx)}. \label{eq:5.2.1}
    \end{align}
    Now $\hat{r}(\xx) = \Pr[y'=-1 | \xx] = \Pr[y'=-1, \xx] / Pr[\xx]$, and similarly $r(\xx) = \Pr[y'=-1, \xx] / \Pr[\xx]$. Hence \begin{equation*}
\frac{\hat{r}(\xx)}{r^*(\xx)} = \frac{\Pr[y'=-1, \xx]}{\Pr[y=-1, \xx]} = \frac{\Pr[y'=-1]}{\Pr[y=-1]}\frac{\ptr(\xx)}{\widehat{\ptr}(\xx)} = 2 \Pr[y'=-1] \frac{\ptr(\xx)}{\widehat{\ptr}(\xx)},
    \end{equation*}
    and analogously
    \begin{equation*}
        \frac{1- \hat{r}(\xx)}{1-r^*(\xx)} = 2 \Pr[y'=1] \frac{\pte(\xx)}{\widehat{\pte}(\xx)}.
    \end{equation*}

    Plugging this back to~\eqref{eq:5.2.1}, we get
    \begin{align*}
    \mathcal{L}_{\hat{r}} - \mathcal{L}_{r^*} & = \frac{1}{2} \E_{\xx \sim \ptr} \log \frac{\ptr(\xx)}{\widehat{\ptr}(\xx)} + \E_{\xx \sim \pte} \log \frac{\pte(\xx)}{\widehat{\pte}(\xx)} + \frac{1}{2} \log \Pr[y'=-1] + \frac{1}{2} \log \Pr[y'=-1] +1 \\
    & =
    \frac{1}{2} D_{KL}(\ptr || \widehat{\ptr}) + D_{KL}(\pte || \widehat{\pte}) + D_{KL}(y || y').
    \end{align*}
\end{proof}

As we discussed above, in order to show this theorem, we will bound the Rademacher complexity of the class $\{ \ell_{\boldsymbol{\theta}}: \|\boldsymbol{\theta}\| \leq B_2, s \leq B_2\}$, where
\begin{equation}
    \ell_{\boldsymbol{\theta},s}(\xx, y) = -\log(1 + \exp(-y(\langle \boldsymbol{\theta}, \xx\rangle + s))).
\end{equation}
We show the following lemma using methods that are now standard, but many of the proofs of similar statements in the literature use stronger assumptions on the distribution of $\xx$ (for example, assuming that almost surely $\|\xx\| \leq C \sqrt{\|\boldsymbol{\Sigma}\|_{op}}$), in order to provide better concentration results. We include a full proof of the following lemma (which needs only second moment assumptions on $\xx$) in~\Cref{sec:rademacher-complexity-log-reg}.
\begin{lemma}
    \label{lem:rademacher-complexity-log-reg}
    Let $\cD$ be a distribution over pairs $(\xx, y)$ with $\xx \in U, y \in \{\pm 1\}$, and let $\boldsymbol{\Sigma} = \E \xx \xx^T$ be a covariance matrix of the marginal $\xx$. Consider a sample $S = ((\xx_1, y_1), \ldots (\xx_n, y_n))$ drawn i.i.d. from $\cD$.  Then
    \begin{equation*}
        \E_{S \sim \cD^n} \sup_{\|\boldsymbol{\theta}\| \leq B, s \leq B} |\widehat{\cL}_{\boldsymbol{\theta}, s}(S) - \cL_{\boldsymbol{\theta}, s}(\cD)|\leq \frac{B \sqrt{\Tr\boldsymbol{\Sigma}}}{\sqrt{n}}.
    \end{equation*}
    where $\cL_{\boldsymbol{\theta}, s}(\cD) := \E_{(\xx, y)\sim \cD} \ell_{\boldsymbol{\theta}, s}$ and $\widehat{\cL}_{\boldsymbol{\theta}, s}(S) := \frac{1}{n} \sum_i \ell_{\boldsymbol{\theta}, s}(\xx_i, y_i)$.
\end{lemma}

Finally, since for a given sample $S$, a function $(\boldsymbol{\theta}, s) \mapsto \widehat{\cL}_{\boldsymbol{\theta}, s}(S)$ is convex, we can efficiently find an approximate minimum of this function. That is,
\begin{fact}
\label{fact:log-reg-optimization}
    Given a sample $S = (\xx_1, y_1), \ldots (\xx_n, y_n)$ and $\varepsilon'$ we can find in polynomial time $\hat{\boldsymbol{\theta}}, \hat{s}$ such that
    \begin{equation*}
        \widehat{\cL}_{\hat{\boldsymbol{\theta}}, \hat{s}}(S) \leq \min_{\boldsymbol{\theta}, s} \widehat{\cL}_{\boldsymbol{\theta}, S} + \varepsilon'.
    \end{equation*}
\end{fact}

We are now ready to prove~\Cref{thm:logistic-regression-result}.

\begin{proof}[Proof of~\Cref{thm:logistic-regression-result}]
Take $\varepsilon'$, depending on $\varepsilon, B_1, B_2, B_3$, which we will fix later. Using~\Cref{fact:log-reg-optimization}, we can find $\hat{\boldsymbol{\theta}}, \hat{s}$ s.t. $\widehat{\cL}_{\hat{\boldsymbol{\theta}}, \hat{s}}(S) \leq \varepsilon' + \min_{\boldsymbol{\theta}, s} \widehat{\cL}_{\boldsymbol{\theta}, s}(S)$. Using Markov inequality for $n \gtrsim {\varepsilon'}^{-2} B_2 \Tr (\str + \ste)$, \Cref{lem:rademacher-complexity-log-reg} implies that with probability $9/10$ simultaneously for all $\boldsymbol{\theta},s$ we have $\cL_{\boldsymbol{\theta}, s}(\cD) = \widehat{\cL}_{\boldsymbol{\theta}, s}(S) \pm \varepsilon'$.
Combining those inequalities, we get
\begin{equation*}
    \cL_{\hat{\boldsymbol{\theta}}, \hat{s}}(\cD) \leq \varepsilon' 
    + \widehat{\cL}_{\hat{\boldsymbol{\theta}}, \hat{s}}(S) \leq 2\varepsilon' 
    + \widehat{\cL}_{\boldsymbol{\theta}^*, s^*}(S) \leq 3 \varepsilon' + \cL_{\boldsymbol{\theta}^*, s^*}(\cD).
\end{equation*}
This implies bound on population regret
\begin{equation*}
     \cL_{\hat{\boldsymbol{\theta}}, \hat{s}}(\cD) - \cL_{\boldsymbol{\theta}^*, s^*}(\cD) \leq 3 \varepsilon',
\end{equation*}
and by~\Cref{lem:population-regret-is-kl}, this implies
\begin{equation*}
    \max(D_{KL}(\pte || \widehat{\pte}), D_{KL}(\ptr || \widehat{\ptr})) \lesssim \varepsilon'.
\end{equation*}
By Pinsker inequality~(\Cref{thm:pinsker}), we can deduce $d_{TV}(\pte, \widehat{\pte}) \lesssim \sqrt{\varepsilon'}$ and similarly $d_{TV}(\pte, \widehat{\ptr}) \lesssim\sqrt{\varepsilon'}$. If $R_2(\ptr || \pte) \leq B_3$, then $\E_{\xx \sim \pte} \pte(\xx)/\ptr(\xx) = \E_{\xx\sim\ptr} (\pte(\xx)/\ptr(\xx))^2 \leq B_3$ and by Markov inequality $\Pr_{\xx \sim \pte}(\pte(\xx) / \ptr(\xx) > B_3 / \varepsilon) \leq \varepsilon$. We can now apply~\Cref{thm:shift-TV}, with $B = B_3 / \varepsilon$ to deduce the desired statement. To this end, we need to choose $\varepsilon' \approx (\varepsilon/B)^2 \lesssim \varepsilon^{4} / B_3^2$, such that the assumption of~\Cref{thm:shift-TV} are satisfied.

This choice of $\varepsilon'$, yields sample complexity used in the logistic regression phase of the argument as $n \approx B_2^2 \Tr (\Sigma_{tr} + \Sigma_{te}) (\varepsilon')^{-2} = B_1 B_2^2 B_3^4 \varepsilon^{-8}$.

We need additional $O(B^2/\varepsilon^2) = O(B_3^2 \varepsilon^3)$ samples to apply estimation in Algorithm~\ref{alg:CV}, which is negligible. The total sample complexity is then $O(B_1 B_2^2 B_3^4 \varepsilon^{-6})$ as desired.
\end{proof}

\subsection{Improved Bounds for Gaussians}\label{sec:LogisticGauss}

\begin{theorem}[\cite{Bblog,OB18}]
    \label{thm:logistic-regression-parameter-closeness}
    Let $(\xx, y)$ follow the logistic model with parameter $\ttheta^*$, s.t. $\|\ttheta^*\| \leq O(1)$. Then using logistic regression by running an empirical minimization of the negative log-likelihood in parameter space $\|\ttheta\| \leq C$, after observing $n \gtrsim d \log d$ samples, the estimated parameter $\hat{\ttheta}$ satisfies \begin{equation*}
        \|\hat{\ttheta} - \ttheta^*\|^2 \leq \frac{d}{n}.
    \end{equation*}
\end{theorem}
In our case, the covariate $\xx$ follows a mixture of two Gaussian distributions. We will show that such a random variable is indeed subgaussian, and therefore we can apply~\Cref{thm:logistic-regression-parameter-closeness}
\begin{fact}
\label{fact:mixture-is-subgaussian}
    Let $\xx$ be distributed according to the mixture of two Gaussian distributions $\xx \sim \frac{1}{2} \cN(\boldsymbol{\mu}_1, \II) + \frac{1}{2} \cN(\boldsymbol{\mu}_2, \II)$, where $\|\boldsymbol{\mu}_1\|, \|\boldsymbol{\mu}_2\| \leq O(1)$. Then $\xx$ is $O(1)$-subgaussian.
\end{fact}
\begin{proof}
Taking $Z = \langle \xx, v\rangle$ for a unit vector $v$, we reduce to one-dimensional case: $Z$ is a mixture of two univariate gaussians.

Then
\begin{align*}
    \E \exp(\lambda(Z - \E Z)) & = \exp(-\lambda(\mu_1 + \mu_2)/2) \frac{\E\exp(\lambda Z_1) + \E\exp(\lambda Z_2)}{2} \\
    & = \exp(-\lambda(\mu_1 + \mu_2)/2) \frac{\exp(\lambda \mu_1 + O(\lambda^2)) + \exp(\lambda \mu_2 + O(\lambda^2)}{2} \\
    & = \exp(O(\lambda^2)) \frac{\exp(\lambda \Delta) + \exp(-\lambda \Delta)}{2},
\end{align*}
where $\Delta = (\mu_1 - \mu_2)/2$. If $\sigma$ is a Rademacher random variable independent of $Z$, then 
\begin{equation*}
    \frac{\exp(\lambda \Delta) + \exp(-\lambda \Delta)}{2} = \E \exp(\lambda \Delta \sigma) = \exp(O(\lambda^2 \Delta^2),
\end{equation*}
using a well-known fact that Rademacher random variables are $O(1)$-subgaussian. This implies
\begin{equation*}
    \E \exp(\lambda(Z - \E Z)) \leq \exp(O(\lambda^2)) \exp(O(\lambda^2 \Delta^2) = \exp(O(\lambda^2)),
\end{equation*}
when $\Delta = O(1)$.
\end{proof}

We can now show that a bound on covariate shift, using $\hat{\ww}(x) := \exp(\langle \hat{\ttheta}, \xx \rangle)$ as an estimate for the true density ratio $\pte(\xx)/\ptr(\xx) = \exp(\langle \ttheta^*, \xx \rangle)$.

\begin{lemma}
    Taking $\hat{\ww}(\xx)$ as above, if $\|\hat{\ttheta} - \ttheta^*\| \leq \varepsilon$, then for any bounded function $\ff$, we have
    \begin{equation*}
        | \E_{\xx\sim \ptr} \hat{\ww}(\xx) \ff(\xx) - \E_{\xx\sim \pte} \ff(\xx)| \lesssim \varepsilon.
    \end{equation*}
\end{lemma}
\begin{proof}
    Taking $w(x) := p_1(x)/p_2(x)$, we can bound 
    \begin{align*}
        | \E_{x\sim p_2} \hat{w}(x) f(x) - \E_{x\sim p_1} f(x)| 
        & \leq \E_{x\sim p_2} |\hat{w}(x) - w(x)|  \\
        & = \E_{x\sim p_2} w(x) \left|1 - \frac{\hat{w}(x)}{w(x)}\right| \\
        & = \E_{x\sim p_1} \left| 1 - \frac{\hat{w}(x)}{w(x)}\right|.
    \end{align*}
    Now, just by definition $\hat{w}(x) / w(x) = \exp(\langle \hat{\boldsymbol{\theta}} - \boldsymbol{\theta}^*, x\rangle)$. Taking $Z := \langle \hat{\boldsymbol{\theta}} - \boldsymbol{\theta}^*, x \rangle$, since $x\sim p_1$ is a multivariate Gaussian with mean $\boldsymbol{\mu}_1$ and covariance $\II$, $Z$ is a univariate gaussian with mean $\E Z = \langle \boldsymbol{\theta}^* - \hat{\boldsymbol{\theta}}, \boldsymbol{\mu}_1\rangle \leq O(\|\boldsymbol{\theta}^* - \hat{\boldsymbol{\theta}}\|_2) \leq O(\varepsilon)$ and variance $\Var(Z) = \|\boldsymbol{\theta}^* - \hat{\boldsymbol{\theta}}\|_2^2 \leq \varepsilon^2$.

    The bound on $\E |1 - \exp(Z)| \leq O(\varepsilon)$ is a simple calculation proven after this proof.
\end{proof}
\begin{lemma}
\label{lem:abs_one_minus_mgf}
    Let $Z$ be a univariate gaussian with $| \E Z | \leq \varepsilon$ and $\Var(Z) \leq \varepsilon^2$. Then
    \begin{equation*}
        \E |1 - \exp(Z)| \lesssim \varepsilon.
    \end{equation*}
\end{lemma}
\begin{proof}
    Let $\mu := \E Z$, $s := \exp(\mu) - 1$, $\tilde{Z} = Z - \E Z$ and $\sigma := \sqrt{\E \tilde{Z}^2}$. We have
    \begin{align*}
        \E | 1 - \exp(Z)| = \E |1 - \exp(\tilde{Z})( 1 + s)| \\
        & = \E |1 - \exp(\tilde{Z})| + s \E |\exp(\tilde{Z})| \\
        & \leq \E | 1 - \exp(\tilde{Z}|) + s \sqrt{\E \exp(2 \tilde{Z})} \\
        & \leq \E | 1 - \exp(\tilde{Z})| + O(\varepsilon).
    \end{align*}

    We can now bound
    \begin{align*}
        \E (1 - \exp(\tilde{Z}))^2 & = 1 + \exp(2 \tilde{Z}) - 2 \exp(\tilde{Z}) \\
        & = 1 + \exp(4 \sigma^2) - 2 \exp(\sigma^2) \\
        & = 1 + (1 + O(\sigma^2)) - 2 (1 + O(\sigma^2))\\
        & \leq O(\sigma^2).
    \end{align*}
    Which gives
    \begin{equation}
        \E | 1 - \exp(\tilde{Z})| \leq \| 1 - \exp(\tilde{Z})\|_2 \leq O(\sigma).
    \end{equation}
Hence $\E | 1 - \exp(Z)| \leq O(\sigma) + O(\varepsilon) = O(\varepsilon)$

\end{proof}

\begin{lemma}
    Let $f$ be a bounded function, and $Z := \hat{w}(\xx) \ff(\xx)$ for $\xx \sim \ptr$ and $\hat{w}$ as above. Then $\Var(Z) \leq O(1)$.
\end{lemma}
\begin{proof}
    Indeed,
    \begin{equation*}
        \Var(Z) \leq \E Z^2 \leq \E_{x\sim p_1} \hat{w}(x) = \E_{x\sim p_1} \exp(\langle 2 \hat{\theta}, x\rangle) = \exp(2 \hat{\theta} \mu_1 + 4 \|\hat{\theta}\|^2) = O(1).
    \end{equation*}
\end{proof}

\subsection{Proof of~\Cref{lem:rademacher-complexity-log-reg} \label{sec:rademacher-complexity-log-reg}}

\begin{lemma}
\label{lem:rademacher-linear}
     The Radeamcher complexity of the class $\mathcal{F} := \{ \xx \mapsto \langle \boldsymbol{\theta}, \xx \rangle : \|\boldsymbol{\theta}\| \leq B \}$ is at most
    $B \sqrt{\Tr \boldsymbol{\Sigma}}/\sqrt{n}$ where $\boldsymbol{\Sigma} := \E_{\xx \sim p} \xx \xx^T$ is the covariance matrix.
\end{lemma}
\begin{proof}
  We can bound the Rademacher complexity of $\cF$ (taking $\sigma_i$ to be independent Rademacher valued random variables) as follows
    \begin{align*}
        n \mathcal{R}_{n, \cD}(\cF) & =\E_{\xx, \sigma} \sup_{\|\boldsymbol{\theta}\| \leq B} \sum_i \sigma_i \langle \xx_i, \boldsymbol{\theta} \rangle \\
        & = \E_{\xx, \sigma} \sup_{\|\theta\| \leq B} \left\langle \sum_i \sigma_i \xx_i, \theta \right\rangle \\
        & = \E_{\xx, \sigma} B \left\|\sum_i \sigma_i \xx_i \right\| \leq B \sqrt{\E\left\|\sum \sigma_i x_i\right\|^2}
    \end{align*}
    Now, since $\E \sigma_i \sigma_j = 0$ for $i\not=j$, we get
    \begin{align*}
    \E \left\|\sum_i \sigma_i x_i \right\|^2 =\sum_i \E \|\xx_i\|^2 = n \Tr \boldsymbol{\Sigma}.
    \end{align*}
    Combining those two, we get
    \begin{equation*}
        \mathcal{R}_{n, \cD}(\cF) \leq B \frac{\sqrt{\Tr \boldsymbol{\Sigma}}}{n},
    \end{equation*}
\end{proof}
Since composition with a Lipschitz function does not increase the Rademacher complexity~(\Cref{thm:rademacher-lipschitz-composition}) we get for any Lipschitz function $\gamma$, the class $\{ \xx \mapsto \gamma(\langle \boldsymbol{\theta}, \xx\rangle) : \|\boldsymbol{\theta}\| \leq B\}$ is also bounded by $B \sqrt{\Tr \boldsymbol{\Sigma}} / \sqrt{n}$.

We can now lift this bound to a bound on the class of relevant loss functions.

\begin{lemma}
    \label{lem:rademacher-y}
    Let $\cF'$ be any family of functions from $U \to \bR$. Consider a family of functions $\cF \subset U \times \{0, 1\} \to \bR$ s.t. for every $f \in \cF$ and every $y \in \{0, 1\}$ we have $f(\cdot, y) \in \cF'$. 
    Let $\cD$ be a distribution over $U \times \{ 0, 1\}$, s.t. $\cD_{\xx}$ is the marginal of $\cD$ on $U$. Then
    \begin{equation*}
        \mathcal{R}_{n, \cD}(\cF) \lesssim \mathcal{R}_{n, \cD_{\xx}}(\cF').
    \end{equation*}
\end{lemma}
\begin{proof}
    We can write arbitrary function $f \in \cF$ as  $f(\xx, y) = y f_1(\xx) + (1-y) f_0(\xx)$, 
    where $f_0, f_1 \in \cF'$. Then
    \begin{equation*}
        \mathcal{R}_{n, \cD}(\cF) = \E_{(\xx, y), \sigma} \sup_{f \in \cF} \sum_i \sigma_i f(\xx_i, y_i) \leq \E_{(\xx, y), \sigma} \sup_{f_1 \in \cF'} \sum_i \sigma_i y_i f_0(x_i) + \E_{(\xx, y), \sigma}\sup_{f_0 \in \cF'} \sum_i \sigma_i (1-y_i) f_1(x_i).
    \end{equation*}

    Now, all we need to show is that for a fixed sample $(\xx_1, y_1), \ldots (\xx_n, y_n)$ we have
    \begin{equation*}
        \E_{\sigma} \sup_{f \in \cF'} \sigma_i y_i f(\xx_i) \leq \E_{\sigma} \sup_{f \in \cF'} \sum \sigma_i f(\xx_i).
    \end{equation*}
    That is a special case of Talagrand's contraction principle~(Theorem 4.12 in \cite{ledoux2013probability}), since for each $i$ we have $|y_i f(\xx_i)| \leq |f(\xx_i)|$.
\end{proof}
Combining~\Cref{lem:rademacher-linear} and~\Cref{lem:rademacher-y}, together with the observation that $t \mapsto \log(1 + \exp(t + s))$ is a 1-Lipschitz function, we obtain the following corollary, where as a reminder, we consider a family of loss functions $\ell_{\theta, s}(\xx, y) := \log (1 + \exp(y\langle \theta, \xx\rangle + s))$.
\begin{corollary}
Family of functions $\cF_{\ell} := \{\ell_{\theta, s}(\xx, y) : \|\boldsymbol{\theta}\| \leq B_2, |s| \leq B_2 \}$, with respect to the distribution $\cD_0$ as in \Cref{sec:logistic}, has Rademacher complexity bounded as
\begin{equation*}
    \mathcal{R}_{n, \cD_0}(\cF_{\ell}) \lesssim B_2 \frac{\sqrt{\Tr (\str + \ste)}}{\sqrt{n}}.
\end{equation*}
\end{corollary}

This, together with~\Cref{thm:rademacher-generalization-bound} and Markov inequality completes the proof of~\Cref{lem:rademacher-complexity-log-reg}.

\section{{Proofs from Section~\ref{sec:RKHS}}}
\subsection{Proof of~\Cref{lem:kernel-logistic-regression-regret-bound} \label{sec:kernel-logistic-regression-regret-bound}}
As usual, instead of minimizing the population loss directly, we minimize the empirical loss. For a sample $S = ((\xx_1, y_1), \ldots (\xx_n, y_n))$ drawn i.i.d. from $\cD_{\boldsymbol{\theta}^*}$ the empirical loss is defined as
\begin{equation*}
\widehat{\mathcal{L}}_{\boldsymbol{\theta}, S} := \frac{1}{n}\sum_i \log(1 + \exp(y_i \boldsymbol{\theta}(\xx_i))),
\end{equation*}

\begin{lemma}
\label{lem:kernel-log-regression-easy}

Given sample $S = ((\xx_1, y_1), \ldots (\xx_n, y_n))$ and desired accuracy $\varepsilon'$ we can efficiently  find  $\hat{\boldsymbol{\theta}}$, s.t.
\begin{equation*}
\label{lem:rkhs-optimization-easy}
    \widehat{\mathcal{L}}_{\hat{\boldsymbol{\theta}}}(S) - \min_{\boldsymbol{\theta}} \widehat{\mathcal{L}}_{\boldsymbol{\theta}}(S) \leq \varepsilon'.
\end{equation*}
\end{lemma}
\begin{proof}
By the Representer Theorem, the optimal $\theta$ can be written as a linear combination
\begin{equation}
\label{eq:kernel-emp-loss}
    \hat{\theta}^* := \sum_i \gamma_i \Phi_{x_i}.
\end{equation}
Finding the minimizer of this form of the empirical loss function~\eqref{eq:kernel-emp-loss} is just a finite-dimensional convex problem
\begin{equation*}
    \min_{\gamma \in \bR^n} \sum_i \log(1 + \exp(y_i \langle \gamma, K_{i} \rangle)),
\end{equation*}
where $K_{i}$ is the $i$-th row of the kernel matrix  $K_{i,j} := K(x_i, x_j)$, and as such it can be solved efficiently.
\end{proof}

When the sample size is large enough, the standard Rademacher-complexity based considerations leads to a generalization bound. The proof is almost identical to the one of~\Cref{lem:rademacher-complexity-log-reg}. Instead of direct calculation for the Rademacher complexity of the class of all bounded linear forms, we need to show the corresponding bound for Rademacher complexity of the kernel classes. Variants of this lemma appear for example in~\cite{bartlett2002rademacher}.
\begin{lemma}
\label{lem:rademacher-kernel}
    If $\cH$ is an RKHS of functions over $U$ with kernel $K : U\times U \to \bR$, and $\mathcal{F} := \{ \boldsymbol{\theta} \in \cH : \|\boldsymbol{\theta}\|_{\cH} \leq B \}$, then
    \begin{equation*}
        \mathcal{R}_{n, \cD}(\mathcal{F}) \leq \sqrt{\E_{\xx\sim \cD} K(\xx, \xx)} B / \sqrt{n}.
    \end{equation*}
\end{lemma}
\begin{proof}
    For given $\xx$, let $\Phi_{\xx} \in \cH$ be such that $\langle \Phi_{\xx}, \theta\rangle_{\cH} = \boldsymbol{\theta}(\xx)\rangle_{\cH}$. Sampling $\xx_1, \ldots \xx_n$ i.i.d. from $\cD$ and, $\sigma_1, \ldots \sigma_n$ independent Rademacher random variables, we have
\begin{align*}
\E_{\xx, \sigma} \sup_{\boldsymbol{\theta}} \sum_i \sigma_i \boldsymbol{\theta}(x_i) & = \E_{x, \sigma} \sup_{\boldsymbol{\theta}} \langle \sum_i \sigma_i \Phi_{x_i}, \boldsymbol{\theta} \rangle \\
& \leq B \E_{x, \sigma} \|\sum_i \sigma_i\Phi_{x_i}\|_{\cH}
\end{align*}
    By Jensen inequality
    \begin{equation*}
        \E_{\xx, \sigma} \|\sum_i \sigma_i\Phi_{x_i}\|_{\cH} \leq \sqrt{\E_{\xx, \sigma}\|\sum_i \sigma_i \Phi_{\xx_i}\|^2} = \sqrt{\E_{\xx} \sum_i \langle \Phi_{\xx_i}, \Phi_{\xx_i}\rangle} = \sqrt{n} \sqrt{\E_{\xx \sim \cD} K(\xx, \xx)}.
    \end{equation*}
    Combining those two we get
    \begin{equation*}
        \frac{1}{n}\E_{\xx, \sigma} \sup_{\theta} \sum_i \sigma_i \theta(x_i) \leq \frac{B}{\sqrt{n}} \sqrt{\E_{\xx \sim \cD} K(\xx, \xx)}.
    \end{equation*}
\end{proof}

The Rademacher complexity bound above, as usual, implies the following uniform bound for the error between population loss an empirical loss of the kernel logistic regression classifier. 

\begin{lemma}
\label{lem:kernel-log-regression-generalization}
    Let $\cH$ be a RKHS with kernel $K$, and assume that $\E_{\xx} K(\xx,\xx) \leq R^2$. Then when $n \gtrsim R/\varepsilon^2$, with probability at least $9/10$ with respect to random sample $S$ we have
    \begin{equation*}
        \sup_{\|\boldsymbol{\theta}\|_{\cH} \leq B} |\mathcal{L}_{\boldsymbol{\theta}}(S) - \widehat{\mathcal{L}}_{\boldsymbol{\theta}}(S)| \leq \varepsilon.
    \end{equation*}
\end{lemma}
\begin{proof}
Combining~\Cref{lem:rademacher-kernel} and~\Cref{lem:rademacher-y}, together with the observation that $t \mapsto \log(1 + \exp(t + s))$ is a 1-Lipschitz function, we get that for an RKHS $\cH$, the family $\{ (\xx, y) \mapsto \ell_{\boldsymbol{\theta}}(\xx, y) : \theta \in \cH, \|\boldsymbol{\theta}\|_{\cH} \leq B\}$ has Rademacher complexity at most $O(B\sqrt{\E_{\xx \sim \cD} K(\xx, \xx)} / \sqrt{n})$, where $\ell_{\boldsymbol{\theta}}(\xx, y) := \log (1 + \exp(\boldsymbol{\theta}(\xx))).$ 
This, together with~\Cref{thm:rademacher-generalization-bound} and Markov inequality completes the proof of~\Cref{lem:rademacher-complexity-log-reg}.
\end{proof}

As usual, those two lemmas together give the concrete regret bound for the kernel logistic regression

\begin{proof}[Proof of \Cref{lem:kernel-logistic-regression-regret-bound}]
We will invoke~\Cref{lem:kernel-log-regression-easy} and~\Cref{lem:kernel-log-regression-generalization} with $\varepsilon' := \varepsilon/3$

We can decompose
\begin{align*}
      \cL_{\hat{\theta}}(\cD_{\theta^*}) - \cL_{\theta^*}(\cD_{\theta^*}) & = (\cL_{\hat{\theta}}(\cD_{\theta^*}) - \widehat{\cL}_{\hat{\theta}}(S) + (\widehat{\cL}_{\hat{\theta}}(S) \\
      & - \widehat{\cL}_{\theta^*}(S)) + (\widehat{\cL}_{\hat{\theta}}(S) - \cL_{\theta^*}(\cD_{\theta^*})).
\end{align*}
By~\Cref{lem:kernel-log-regression-generalization}, the first and the last summand here are each at most $\varepsilon'$. On the other hand, by~\Cref{lem:kernel-log-regression-easy}, we have $\widehat{\cL}_{\hat{\theta}}(S) - \widehat{\cL}_{\theta^*}(S) \leq \varepsilon'$, leading to
\begin{equation*}
    \cL_{\hat{\theta}}(\cD_{\theta^*}) - \cL_{\theta^*}(\cD_{\theta^*}) \leq 3 \varepsilon' \leq \varepsilon.
\end{equation*}
\end{proof}

\subsection{Proof~of~\Cref{thm:covariate-shift-kernel-logistic-regression} \label{sec:covariate-shift-kernel-logistic-regression}}
\begin{proof}

With~\Cref{lem:kernel-logistic-regression-regret-bound}, this proof is identical to the proof of~\Cref{thm:logistic-regression-result}.

Consider a distribution $\cD$ given by sampling $y \in \{\pm 1\}$, and then $x$ from $\ptr$ if $y=1$ and $x$ from $\pte$ if $y=-1$.

We can apply the kernel logistic regression to a random sample from $\cD$~--- according to~\Cref{thm:covariate-shift-kernel-logistic-regression}, we will be able to find a function $\widehat{\theta} \in \cH$ with bounded regret with respect to distribution $\cD$. According to~\Cref{lem:population-regret-is-kl}, this $\widehat{\theta}$ is of form $\ln(\widehat{\pte}/\widehat{\ptr})$ for some $\widehat{\ptr}$ and $\widehat{\pte}$ both $(\varepsilon')^2$-close in KL-divergence to $\ptr$ and $\pte$ respectively. Using now Pinsker we get the TV-distance bounds on $d_{TV}(\widehat{\ptr}, \ptr) \leq \varepsilon'$ and $d_{TV}(\widehat{\pte}, \pte) \leq \varepsilon'$. The bound $R_2(\ptr || \pte) \leq B_2$, together with Markov inequality implies that $\Pr_{\xx \sim \pte}[\pte(\xx)/\ptr(\xx) > B] \leq \varepsilon$ for $B = B_2/\varepsilon$. 

We can now apply~\Cref{thm:shift-TV} with this $B$ to deduce that we can use the truncated version of $\exp(\theta(x))$ as importance weights to solve the covariate shift problem; to this end we need $\varepsilon' = \varepsilon/B = \varepsilon^2 / B_2$; so the desired regret error in~\Cref{lem:kernel-logistic-regression-regret-bound} is of order~$\varepsilon^4/B_2^2$, and according to this lemma we can get it with $O(D B_1^2 B_2^4 /\varepsilon^8)$ samples.
\end{proof}

\section{Discussing  KMM guarantees}\label{sec:KMMComparison}
\paragraph{Description of the Kernel Mean Matching algorithm}
The KMM algorithm in the realizable scenario assumes that the unknown function $\ff$ is bounded in some RKHS $\cH$, $\|\ff\|_{\cH}\leq M$. If the kernel corresponding to $\cH$ is also bounded, say $K(\xx, \xx) \leq 1$ for all $\xx$. 

In this case, with enough samples we can guarantee that $\|n_{tr}^{-1} \sum_i \Phi_{\xx_i} - \E_{\xx \sim \ptr} \Phi_{\xx}\|_{\cH} \leq \varepsilon/M$ (and similarly for $\tilde{\xx} \sim \pte$). Then by solving a linear program, we can find some weights $\hat{\beta}_i$ s.t. the empirical averages of the reproducing kernel map are close after reweighting, i.e., $\|\sum_i \hat{\beta}_i \Phi_{\xx_i}/n_{tr} - \sum_{i} \Phi_{\tilde{\xx}_i} /n_{te}\|_{\cH} \leq \varepsilon/M$ (these weights exist, because in particular $\beta_i = \pte(\xx_i) / \ptr(\xx_i)$ is a feasible solution). Now it is easy to show that $\hat{\beta}$ can be used as importance weights to solve the covariate-shifted mean estimation:
\begin{align*}
    \sum_i \frac{1}{n} \hat{\beta}_i \ff(\xx_i)&  = \langle \sum_i \frac{1}{n} \hat{\beta}_i \Phi_{\xx_i}, \ff \rangle_{\cH} \\
    & = \langle \E_{\tilde{\xx}\sim \pte} \Phi_{\tilde{\xx}} , \ff \rangle_{\cH} + \langle \Delta, \ff \rangle_{\cH} \\
    & = \E_{\tilde{\xx}\sim \pte} \ff(\xx) \pm \|\Delta\|_{\cH} \|\ff\|_{\cH},
\end{align*}
where $\Delta := \sum_i n^{-1} \hat{\beta}_i \Phi_{\xx_i} - \E_{\tilde{\xx}\sim \pte} \Phi_{\tilde{\xx}}$ satisfies $\|\Delta\|_{\cH} \leq \varepsilon / M$.

As discussed above, collecting $n_{tr} \gtrsim (MB/\varepsilon)^2$ samples from $\ptr$ and $n_{te} \gtrsim (M/\varepsilon)^2$ samples from $\pte$ is enough to guarantee the corresponding closeness of three relevant empirical averages to their population averages. 

\paragraph{Analysis of the plug-in method}
Alternative, naive approach, to solve the covariate shift mean estimation problem in the same setup as above, is to find $\hat{\ff}$ with a sufficiently small distance $L_1$ from $\ff$ with respect to $\ptr$, and use it as a proxy for $\ff$ when computing the averages with respect to $\pte$. 

\begin{lemma}
    When $\ff \in \cH$ is bounded by $\|\ff\| \leq M$, using $O((M/\varepsilon')^2)$ samples $(\xx, \ff(\xx))$ for $\xx\sim\ptr$ we can find a function $\hat{\ff}$ s.t.
    \begin{equation*}
        \E_{\xx \sim \ptr} |\ff(\xx) - \hat{\ff}(\xx)| \leq \varepsilon'.
    \end{equation*}
\end{lemma}
\begin{proof} 
    Consider the family of function $\cF' := \{ \xx \mapsto |\boldsymbol{h}(x)| : \boldsymbol{h} \in \cH, \|\boldsymbol{h}\| \leq 2M \}$. Since $\xx \mapsto |\xx|$ is 1-Lipschitz, and we have a Rademacher complexity bound for a radius $M$-ball in the kernel space~(\Cref{lem:rademacher-kernel}), the Rademacher complexity of $\mathcal{R}_{n}(\cF') \leq O(M/\sqrt{n})$.
    
    Using~\Cref{thm:rademacher-generalization-bound} we obtain a generalization bound: after collecting $n  = O(M/\varepsilon')^2$ samples, with probability $9/10$ we get a uniform bound on the generalization error for every $\boldsymbol{h} \in \cH$ with $\|\boldsymbol{h}\| \leq 2$.  
    \begin{equation*}
        \sup_{\|\boldsymbol{h}\| \leq 2M} \frac{1}{n} \sum_i |\boldsymbol{h}(\xx_i)| = \E_{\xx} |\boldsymbol{h}(\xx)| \pm \varepsilon'/2.
    \end{equation*}
    In particular, taking $\boldsymbol{h}$ of form $\ff - \ff'$ for some $\|\ff'\|_{\cH} \leq M$, we get
        \begin{equation*}
        \sup_{\|\ff'\| \leq M} \frac{1}{n} \sum_i |\ff(\xx_i) - \ff'(\xx_i)| = \E_{\xx} |\ff'(\xx) - \ff(\xx)| \pm \varepsilon'/2.
    \end{equation*}
    It is enough therefore to approximately minimize the empirical loss: given samples $(\xx_i, \ff(\xx_i))$, we need to find 
    \begin{equation*}
        \min_{\|\ff'\|_{\cH} \leq M} \frac{1}{n} \sum_{i} |\ff'(\xx_i) - \ff(\xx_i)|.
    \end{equation*}
    By the Representer Theorem~(\Cref{thm:representer}), this minimum is realized by some $\ff'$ of form $\sum_{i} \Phi_{\xx_i} \lambda_i$, and the problem becomes just a minimization
    \begin{equation*}
        \min_{\lambda \in \bR^n} \| K \lambda - y\|_1,
    \end{equation*}
    where $K$ is a kernel matrix $K_{ij} = K(\xx_i, \xx_j)$, and $y$ is a vector of observations $y_i = \ff(\xx_i)$. This minimization is just a linear program and can be solved efficiently.
\end{proof}

Note that if we now find $\hat{\ff}$ with $L_1$ error $\varepsilon':= \varepsilon/B$ with respect to $\ptr$, since the density ratio $\pte/\ptr$ is bounded by $B$ we can also bound the $L_1$ error with respect to $\pte$. This clearly implies closeness of the expectations
\begin{equation*}
    \E_{\xx \sim \pte} \ff(\xx) - \E_{\xx \sim \pte} \hat{\ff}(\xx) \leq \E_{\xx \sim \pte} |\ff(\xx) - \hat{\ff}(\xx)| \leq B \E_{\xx \sim \ptr} |\ff(\xx) - \hat{\ff}(\xx)| \leq \varepsilon.
\end{equation*}

We can now use $(M/\varepsilon)^2$ samples from $\pte$ to estimate $\E_{\xx \sim \pte} \hat{\ff}(\xx)$ by a sample average. This leads to the final sample complexity $n_{tr} \gtrsim (MB/\varepsilon)^2$ samples from $\ptr$ and $n_{te} \gtrsim (M/\varepsilon)^2$ samples from $\pte$, matching the guarantees of the KMM method.

\end{document}